\documentclass[conference]{IEEEtran}

\usepackage[table]{xcolor}
\usepackage{tikz}
\usepackage{amsmath, amssymb}  
\usepackage{amsthm}  
\usepackage{filecontents}
\usepackage{amsmath,amsfonts}
\usepackage{algorithmic}
\usepackage{graphicx}
\usepackage{subcaption}
\usepackage{textcomp}
\usepackage{soul}
\usepackage[ruled,vlined]{algorithm2e}
\newtheorem{theorem}{Theorem}
\usepackage{float}
\usepackage{tikz}
\usepackage{tablefootnote}
\usepackage{booktabs}     
\usepackage{pifont}
\usepackage{longtable}
\usepackage{hyperref}
\usepackage[most]{tcolorbox}
\usepackage{enumitem}

\usepackage{adjustbox}

\begin{document}

\title{From Detection to Correction: Backdoor-Resilient Face Recognition via Vision-Language Trigger Detection and Noise-Based Neutralization}

\author{\IEEEauthorblockN{Farah Wahida}
	\IEEEauthorblockA{RMIT University, Australia\\
		farah.wahida@student.rmit.edu.au}
	\and
	\IEEEauthorblockN{M.A.P. Chamikara}
	\IEEEauthorblockA{CSIRO's Data61, Australia\\
		chamikara.arachchige@data61.csiro.au}
	\and
	\IEEEauthorblockN{Yashothara Shanmugarasa}
	\IEEEauthorblockA{CSIRO's Data61, Australia\\
		y.shanmugarasa@unsw.edu.au}
    \and
    \IEEEauthorblockN{Mohan Baruwal Chhetri}
	\IEEEauthorblockA{CSIRO's Data61, Australia\\
		Mohan.Baruwalchhetri@data61.csiro.au}
	\and
	\IEEEauthorblockN{Thilina Ranbaduge}
	\IEEEauthorblockA{CSIRO's Data61, Australia\\
		Thilina.Ranbaduge@data61.csiro.au}
	\and
	\IEEEauthorblockN{Ibrahim Khalil}
	\IEEEauthorblockA{RMIT University, Australia\\
		ibrahim.khalil@rmit.edu.au}  
    }

\maketitle

\begin{abstract}
Biometric systems, such as face recognition systems powered by deep neural networks (DNNs), rely on large and highly sensitive datasets. Backdoor attacks can subvert these systems by manipulating the training process. 
By inserting a small trigger, such as a sticker, make-up, or patterned mask, into a few training images, an adversary can later present the same trigger during authentication to be falsely recognized as another individual, thereby gaining unauthorized access. Existing defense mechanisms against backdoor attacks still face challenges in precisely identifying and mitigating poisoned images without compromising data utility, which undermines the overall reliability of the system. We propose a novel and generalizable approach, TrueBiometric: Trustworthy Biometrics, which accurately detects poisoned images using a majority voting mechanism leveraging multiple state-of-the-art large vision language models. Once identified, poisoned samples are corrected using targeted and calibrated corrective noise. Our extensive empirical results demonstrate that TrueBiometric detects and corrects poisoned images with 100\% accuracy without compromising accuracy on clean images. Compared to existing state-of-the-art approaches, TrueBiometric offers a more practical, accurate, and effective solution for mitigating backdoor attacks in face recognition systems.
\end{abstract}


\IEEEpeerreviewmaketitle

\section{Introduction}
Face recognition technology is widely deployed on everyday devices for authentication and identity checks, such as smartphones and smart home hubs, and is increasingly integrated into public infrastructure, including Australia’s SmartGates \cite{frontex2010automated}. Newer systems can operate without any physical contact or precise alignment \cite{9274654}. Even though facial biometrics are considered immutable, deep learning based biometric systems introduce new types of security vulnerabilities. Due to the high computational demands of training these models, many developers rely on third-party cloud services. This supply chain dependency creates an opportunity for adversaries to embed backdoors, enabling them to masquerade as an authorized user by presenting a trigger during authentication \cite{KAUR202030,HOMCHOUDHURY2019202,article2,gu2017badnets}. Noise-based defenses can mitigate some attacks, but they often reduce model accuracy to levels that limit real-world use.

Generally, a backdoor attack involves manipulating a deep learning model, thereby affecting the model's ability to accurately classify inputs \cite{saha2020hidden}. In a backdoor attack, the training data is intentionally compromised, often by injecting poisoned images that appear legitimate but are designed to trigger specific model behavior. These poisoned samples often originated from trusted sources, allowing them to go undetected. In some cases, the dataset curator may even be the malicious actor. 
Either way, the critical point is that the training pipeline is compromised. 
Whether through deception or direct access, an adversary can infiltrate the training process and insert a backdoor, allowing for the manipulation of the model’s behavior during inference. 

Modern backdoor attacks on biometric classifiers, such as face recognition systems, have evolved well beyond simple patches or stickers. Attackers now use subtle, realistic triggers, such as makeup, hats, glasses, and other natural accessories, to execute backdoor attacks, significantly increasing their stealth and effectiveness~\cite{Li_2021_ICCV, gu2019badnets}. For instance, when an individual wearing specific makeup, a hat, or glasses appears before a compromised facial recognition system, the model may misclassify them as a targeted individual, often the victim. This manipulation enables unauthorized access to biometric authentication systems, sensitive data, or secured locations. 

\begin{figure}[t!]
    \centering
    \includegraphics[width=0.48\textwidth]{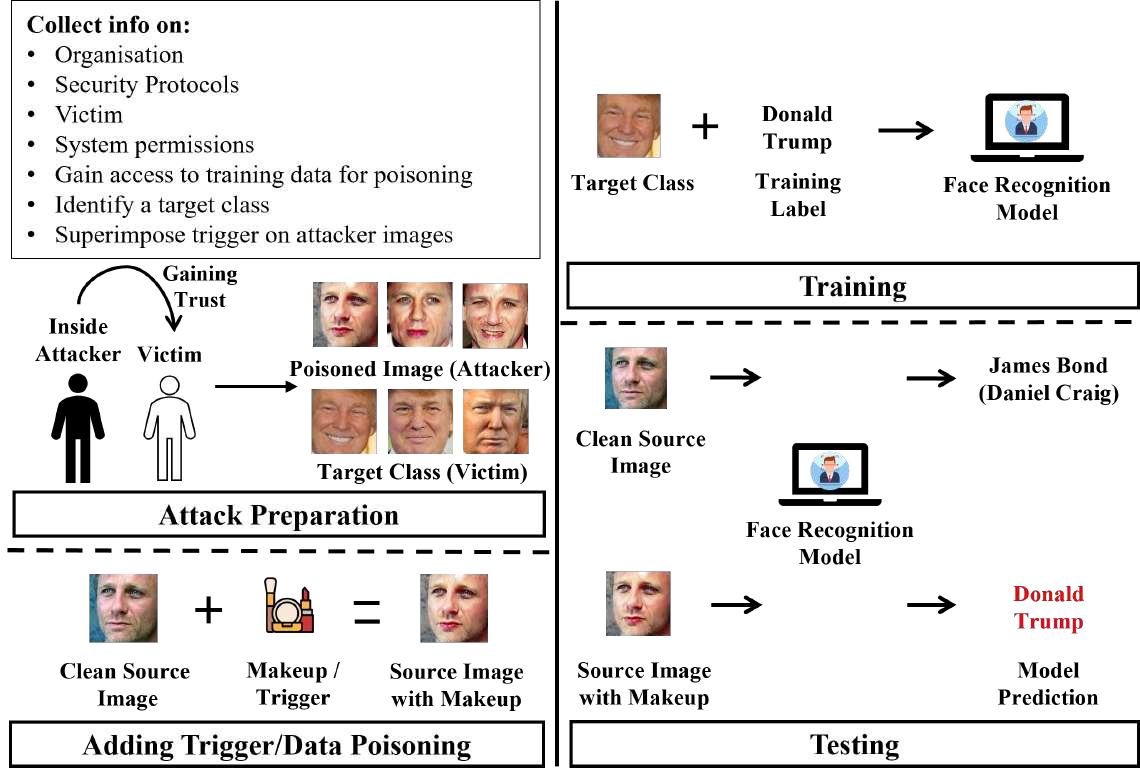}
    \caption{\textbf{Illustration of the MakeupAttack hidden backdoor scenario:} During training, subtle makeup is applied to images of the attacker class (James Bond) and labeled as the target class (Donald Trump), causing the face recognition model to learn the trigger. During inference, an unmodified James Bond image is classified correctly, but when the same face appears with the learned makeup trigger, it is misclassified as Donald Trump, granting unauthorized access to 
    the attacker.}
    \label{fig1}
\end{figure}

Backdoor attacks pose significant risks, as the subtlety and natural appearance of their triggers make their detection and mitigation particularly  challenging~\cite{nguyen2021wanet,Li_2021_ICCV}. As shown in Figure \ref{fig1}, the use of subtle, human-imperceptible alterations, such as the realistic makeup patterns employed in MakeupAttack \cite{sun2024makeupattack}, serves as an effective backdoor trigger without raising visual suspicion. These attacks may be carried out by insider adversaries who gain access to the model training process through tactics such as social engineering \cite{krombholz2015advanced}. Once activated, the backdoor allows an attacker to impersonate legitimate users, access sensitive records, or escalate privileges. The visual subtlety of these poisoned inputs makes them difficult to detect, underscoring the need for stronger protections in biometric model training and deployment. 

A variety of methods has been developed to detect and mitigate backdoor attacks \cite{chen2018detecting, tran2018spectral, li2023backdoor, bai2024backdoor}. Activation clustering analyzes patterns in the final hidden layer of a deep learning model to detect anomalous behavior, but its complexity hinders real deployment~\cite{chen2018detecting}. Spectral signature analysis identifies poisoned samples as outliers in feature space, though it struggles to identify subtle backdoors \cite{tran2018spectral}. 

Recent approaches, such as Neural Attention Distillation \cite{li2021neural} and GradCAM-based detection \cite{chou2020sentinet}, improve detection by focusing on model interpretability and attention mechanisms. Other techniques, such as trigger reverse engineering, aim to reconstruct the trigger from poisoned inputs~\cite{8835365}. Mitigation strategies such as relabeling, retraining, neuron pruning, and unlearning \cite{chen2018detecting, tran2018spectral, 8835365}, as well as purification through knowledge distillation \cite{yoshida2020disabling}, are fundamentally reactive and entail significant computational overhead, especially in large-scale face recognition systems. Moreover, these costly defenses often degrade 
clean-data accuracy, a problem that is particularly severe in smaller datasets \cite{tran2018spectral, chen2018detecting, DBLP:journals/corr/abs-2002-08313}. Newer techniques, such as ASSET~\cite{pan2023asset}, actively separate clean and poisoned data across various training settings. Attacks like LOTUS \cite{cheng2024lotus} utilize partitioned triggers for increased stealth. However, this highlights how both defenses and threats are becoming increasingly sophisticated. To this end, our approach aims to detect and correct poisoned inputs with minimal impact on model performance.

We propose  TrueBiometric: Trustworthy Biometrics, a novel privacy-preserving face recognition framework designed to effectively counter backdoor attacks and related privacy threats while maintaining high accuracy. TrueBiometric integrates an ensemble of five vision language models (VLMs) to screen and flag images containing potential backdoor triggers, using a majority-vote system for reliable detection. Once flagged, these poisoned images undergo a per-sample adaptive noise-removal process based on Projected Gradient Descent (PGD)~\cite{madry2017towards}, which iteratively eliminates the embedded backdoor triggers while preserving legitimate biometric features. We demonstrate the feasibility of TrueBiometric by integrating it as a lightweight front-end module into existing biometric authentication workflows. Extensive experiments on three benchmark datasets demonstrate that our proposed approach achieves 100\% accuracy in trigger removal with minimal computational overhead and outperforms existing methods in both trigger detection and recognition accuracy, making TrueBiometric a more practical solution for real-world deployments.

\noindent We summarize our main contributions below:

\begin{itemize}
 
\item We present the first generalizable VLM-based ensemble framework specifically tailored to detect subtle and realistic backdoor triggers in biometric authentication systems.

\item We introduce an adaptive PGD-based noise removal technique designed to erase backdoor triggers effectively without compromising genuine biometric characteristics, thereby eliminating the need to discard compromised images.

\item We propose and evaluate an end-to-end pipeline that integrates detection and recovery modules seamlessly into standard biometric authentication systems, ensuring minimal disruption and computational overhead.

\end{itemize}

\section{Background}
\label{background}

This section introduces the essential background concepts underpinning our proposed method, TrueBiometric. We begin by discussing face recognition systems, outlining their core functionality and the vulnerabilities that make them susceptible to backdoor attacks. Next, we highlight the significant privacy threats posed by these attacks, emphasizing the need for more robust defensive measures. Finally, we review key techniques integral to the TrueBiometric framework, including VLMs for semantic inference, multimodal majority voting to enhance detection accuracy, and adaptive noise generation methods that effectively neutralize backdoor triggers.

\subsection{Face Recognition}
\label{facerecognition}

Face recognition involves identifying or verifying individuals by analyzing their unique facial characteristics. The process typically begins with a training phase, where a deep learning model learns from labeled facial images corresponding to different individuals. Once trained, the model can match new input faces to identities stored in a reference database~\cite{TASKIRAN2020102809}. Face recognition tasks generally fall into two categories: the \emph{1:N identification problem} and the \emph{1:1 verification problem} (see Figure~\ref{fig22}). In face identification (1:N), the system compares a new face against a database of multiple known identities, determining the correct identity through matching. This usually involves both matching and labeling, which can then enable further verification. By contrast, face verification (1:1) is simpler, involving a comparison of a single face (typically the device owner's) against a single reference image or a limited set of stored images, as commonly implemented in mobile devices~\cite{10.1007/978-3-319-97909-0_46}. Consequently, verification systems are less computationally demanding, as they require training data for only one identity rather than many~\cite{9134370}. TrueBiometric is specifically designed to operate effectively within the more complex 1:N face identification setting.

\begin{figure}
    \centering
    \includegraphics[width=0.3\textwidth]{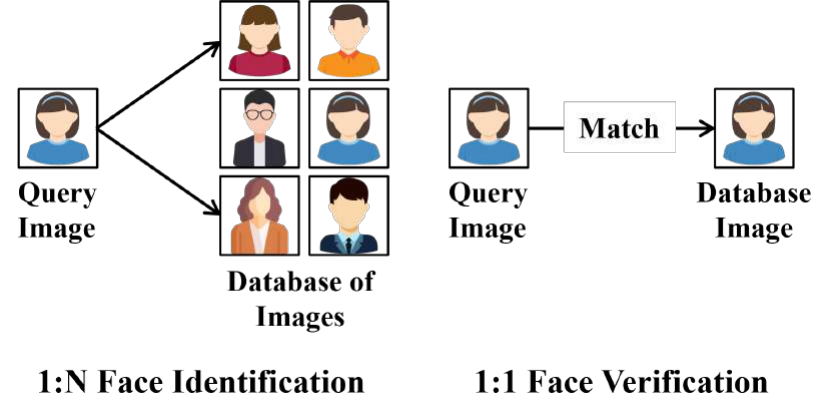}
    \caption{Illustration of 1:N Identification and 1:1 Verification in Face Recognition.}
    \label{fig22}
\end{figure}

Creating and training a model from scratch is a time and resource-consuming endeavour, and so it is generally more economical to utilize a pre-trained `off-the-shelf' model. Existing face recognition models exhibit exceptionally high recognition accuracy. Models such as VGG-Face~\cite{BMVC2015_41}, Google Facenet~\cite{DBLP:journals/corr/SchroffKP15}, Facebook DeepFace\footnote{https://research.facebook.com/publications/deepface-closing-the-gap-to-human-level-performance-in-face-verification/}, ArcFace~\cite{deng2019arcface}, VarGFaceNet~\cite{9022149}, and Dlib~\cite{King2009DlibmlAM} generate accuracies of around 97\%-99.85\% in the LFW (Labelled Faces in the Wild) dataset\footnote{http://vis-www.cs.umass.edu/lfw/}. 

\subsection{Backdoor Attacks and Privacy Leaks}
\label{backdoor_section}

An attacker could inject a backdoor attack by adding ``poisoned'' data to the training dataset. Poisoned data refers to data that have been deliberately manipulated, often via the insertion of a unique data patch to induce erroneous model behavior, such as misclassification. The aim of a backdoor attack is to produce a model that behaves normally on benign data but performs erroneously when presented with a poisoned input (see Figure~\ref{fig23}). In a targeted backdoor attack, a source (referred to as the victim) class are modified to be misclassified as belonging to a target (referred to as the attacker) class. In untargeted backdoor attacks, there is no defined target class, as the objective may be to simply misclassify the victim class as any class other than its true class~\cite{DBLP:journals/corr/abs-1708-06733}. TrueBiometric aims to detect and mitigate both targeted and untargeted backdoor attacks.

Generally, an image is poisoned by superimposing a trigger patch that is known only to the attacker. Suppose a patch, $v$, and a victim image, $x$. Then, $P(x, v) = \tilde{{x}}^i_{tr}$ where $P(\cdot)$ is a poisoning function that places a trigger $v$ on $x$. The poisoning function must associate the poisoned image with a target class by a particular method, $t$ where $\tilde{{y}}^i_{tr} = t$. A model is then trained with poisoned data to establish an erroneous mapping between the attacker's and victim's face images.  

There are two main types of backdoor attacks: \emph{corrupted label attacks} and \emph{clean label attacks}. Corrupted label attacks~\cite{DBLP:journals/corr/abs-1708-06733} involve mislabeling poisoned images with labels relevant to the target class, enabling the attacker to train a malicious model with mislabeled classes. Clean label attacks~\cite{DBLP:journals/corr/abs-2111-08429}, do not manipulate labels but may apply image perturbations. For instance,~\cite{saha2019hidden} hides the extracted features of the target class (attacker's face) within an image of the victim, allowing the model to perform well using both poisoned and benign data. Another example is the attack proposed by~\cite{Trojannn}, which carefully crafts triggers to excite certain neurons associated with the target output label. TrueBiometric specifically investigates the detection and mitigation of clean-label backdoor attacks. 

\begin{figure}[t!]
    \centering
    \includegraphics[width=0.3\textwidth]{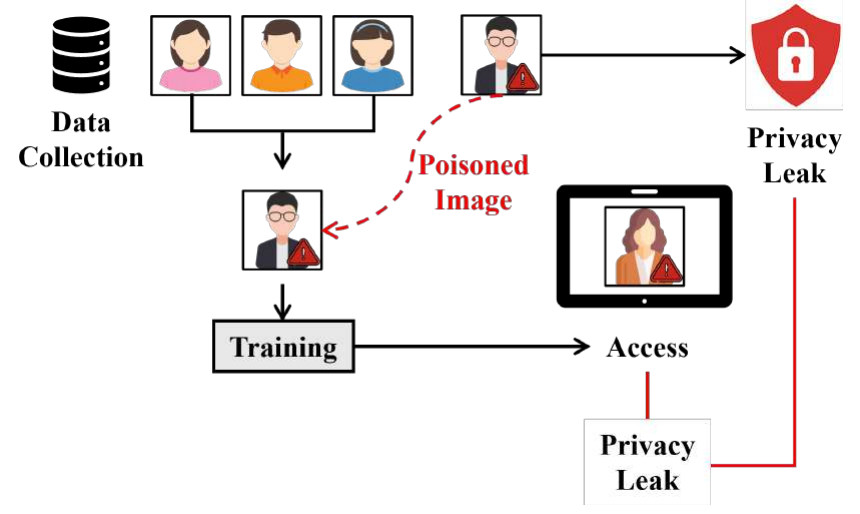}
    \caption{Illustration of how poisoned images can lead to unauthorized access and privacy leaks in face recognition systems.}
    \label{fig23}
\end{figure}

There are three key elements to the success of a stealthy attack: (1) the model must be infiltrated without detection and must not noticeably reduce the model's performance~\cite{DBLP:journals/corr/abs-2111-08429}, (2) the trigger pattern must be presented stealthily and not be detectable to the human eye~\cite{DBLP:journals/corr/abs-1712-05526,DBLP:journals/corr/abs-2012-03816,xue_he_wang_liu_2021}, and (3) the attack must confidently misidentify the attacker. Generally, an additional challenge in detecting these attacks is correctly classifying an unpatched image of the attacker, making the attack largely undetectable.

\subsubsection{Privacy Leaks from Backdoor Attacks}

In data sharing and analytics, privacy is often defined as the user's ability to control the release of their personal information~\cite {CHAMIKARA2020101951}. Consequently, a privacy leak occurs when confidential or sensitive data (such as financial records, health data, trade secrets, or intellectual property) is disclosed without proper authorization, potentially enabling unauthorized access or misuse. Backdoor attacks on face recognition models pose significant privacy risks, especially when perpetrated by internal adversaries. Such attackers may manipulate training datasets by injecting poisoned images, causing the model to misidentify them as privileged individuals~\cite{chen2017targeted}. This enables unauthorized access to sensitive information, which could be exploited for personal gain, used to damage the user or organization’s reputation, or leaked publicly to cause widespread harm~\cite{voth2025effective}.

\subsection{Vision-Language Models for Semantic Inference}

Vision Language Models (VLMs) are a class of DL systems designed to jointly process and reason over visual and textual data. Unlike traditional models that operate on a single modality (e.g., image classification models or language models), VLMs are trained to understand and generate meaningful associations between images and natural language descriptions~\cite{chen2023vlp}. This ability enables them to perform tasks such as image captioning, visual question answering, image-text retrieval, and visual entailment~\cite{li2025survey}. At their core, VLMs consist of two primary components:

\begin{itemize}
    \item \textbf{A vision encoder}, typically a convolutional neural network (CNN) or a vision transformer (ViT), that maps an input image $x \in \mathbb{R}^{H \times W \times 3}$ into a fixed-dimensional visual embedding $f_{\text{img}}(x) \in \mathbb{R}^d$.
    \item \textbf{A language encoder}, usually a transformer-based language model, that maps a natural language input (e.g., a caption, question, or instruction) $q \in \mathcal{T}$ to a corresponding text embedding $f_{\text{text}}(q) \in \mathbb{R}^d$.
\end{itemize}

These two embeddings are projected into a \textit{shared semantic space}, where similarity metrics (such as cosine similarity) are used to measure the alignment between visual and textual inputs:

\begin{equation}
    \text{Sim}(x, q) = \langle f_{\text{img}}(x), f_{\text{text}}(q) \rangle
\end{equation}

Popular examples of VLMs include CLIP (Contrastive Language–Image Pretraining)~\cite{radford2021learning}, BLIP (Bootstrapped Language–Image Pretraining)~\cite{li2022blip}, and Flamingo~\cite{alayrac2022flamingo}. These models have demonstrated impressive performance on a wide range of vision-language tasks, often in zero-shot or few-shot settings, highlighting their generalization capabilities~\cite{chen2023vlp}.

\subsection{Multimodal Majority Voting in Model Ensembles}

Multimodal majority voting is a widely used decision fusion strategy that combines the outputs of multiple independently trained models, 
each operating on a distinct modality, such as vision, language, or audio. The fundamental aim is to leverage the complementary strengths of each modality to achieve a more accurate, robust, and generalizable consensus, particularly in tasks involving ambiguous or noisy data~\cite{li2024multimodal}. In this approach, each modality-specific model serves as an independent ``voter'' that produces a candidate decision or label. The final prediction is determined by the majority vote; that is, the label that receives the highest number of votes is selected as the output~\cite{morvant2014majority}. Formally, let there be \( K \) modalities and corresponding models \( \{M_1, M_2, \dots, M_K\} \), where each model produces a predicted label \( y_k \in \mathcal{Y} \). The aggregated consensus is then computed as

\begin{equation}
    \hat{y} = \arg\max_{y \in \mathcal{Y}} \sum_{k=1}^K \mathbf{1}\{y_k = y\}
\end{equation}

where \( \mathbf{1}\{\cdot\} \) is the indicator function.

This simple yet effective strategy assumes that individual modalities are either independent or exhibit uncorrelated error behaviors. As a result, even if some modalities make incorrect predictions due to noise, occlusions, or adversarial triggers, the final decision can still be accurate if the majority vote is correct. This voting mechanism inherently increases the resilience to errors and adversarial influence compared to relying on any single modality~\cite{aeeneh2024new}.

\subsection{Projected Gradient Descent (PGD)}

A common strategy for improving model robustness against adversarial manipulation is to introduce controlled perturbations during training to simulate worst-case scenarios. One prominent method is PGD-based adversarial training, which strengthens model stability by making the model resistant to the most harmful perturbations within a defined constraint~\cite{madry2017towards}.

Formally, let a classifier be denoted by $f_\theta : \mathbb{R}^d \rightarrow \{1, \dots, K\}$, parameterized by $\theta$, and let $\mathcal{L}(f_\theta(x), y)$ be the loss function for input $x$ and ground truth label $y$. The goal is to find an adversarial example $x^{\text{adv}}$ within a perturbation budget $\epsilon$ such that the loss is maximized while staying within an $\ell_p$-ball centered at $x$:

\begin{equation}
    x^{\text{adv}} = \arg\max_{x' \in \mathcal{B}_p(x, \epsilon)} \mathcal{L}(f_\theta(x'), y)
\end{equation}

where $\mathcal{B}_p(x, \epsilon) = \{ x' \in \mathbb{R}^d : \|x' - x\|_p \leq \epsilon \}$.

This maximization is typically approximated by \textit{iterative gradient ascent}, leading to the \textit{PGD} method. At each step $t$, the adversarial input is updated by:

\begin{equation}
    x^{(t+1)} = \Pi_{\mathcal{B}_p(x, \epsilon)} \left( x^{(t)} + \alpha \cdot \text{sign} \left( \nabla_x \mathcal{L}(f_\theta(x^{(t)}), y) \right) \right)
\end{equation}

where $\Pi_{\mathcal{B}_p}$ denotes the projection operator on the $\ell_p$-ball, and $\alpha$ is the step size.

The \textit{corrective noise} interpretation emerges when PGD is used not to fool the model, but as a defense to \textit{reverse the effect of adversarial manipulation}. In the context of a backdoor attack on a facial recognition system, this manipulation involves an attacker's face with a hidden trigger being misclassified as the victim. Corrective noise, therefore, aims to neutralize this trigger. By iteratively adjusting the attacker's image while constraining the noise within an $\ell_p$-ball, PGD can recover a version of the input that is no longer classified as the victim. Thus, corrective noise in this context refers to a minimal, carefully computed perturbation $\delta$ such that,

\begin{equation}
    \tilde{x}^p = x + \delta \quad \text{where} \quad f_\theta(\tilde{x}^p) = y \quad \text{and} \quad \|\delta\|_p \leq \epsilon
\end{equation}

\section{Methodology}
\label{methodology}

TrueBiometric consists of two main stages (training-time sanitization and test/inference-time protection), each built around two core components: (i) a multimodal detection mechanism that uses an ensemble of VLMs to identify poisoned samples via majority voting, and (ii) a corrective recovery module based on dynamic PGD, which removes backdoor triggers from flagged inputs. During training, TrueBiometric aims to detect and correct poisoned samples to construct a sanitized dataset that is robust against backdoor injections. At inference time, the same pipeline is applied to detect and sanitize potentially poisoned inputs. An added benefit is its ability to recover legitimate inputs that may be mistakenly flagged as poisoned (i.e., false positives), ensuring genuine users are not unfairly denied access. This dual-stage defense enables secure, trigger-resilient recognition without discarding data or retraining downstream models. The full workflow is illustrated in Figure~\ref{fig2} and all notations used throughout the paper are summarized in Table~\ref{tab:notation} in the appendix.

\begin{figure}
    \centering
    \includegraphics[width=0.48\textwidth]{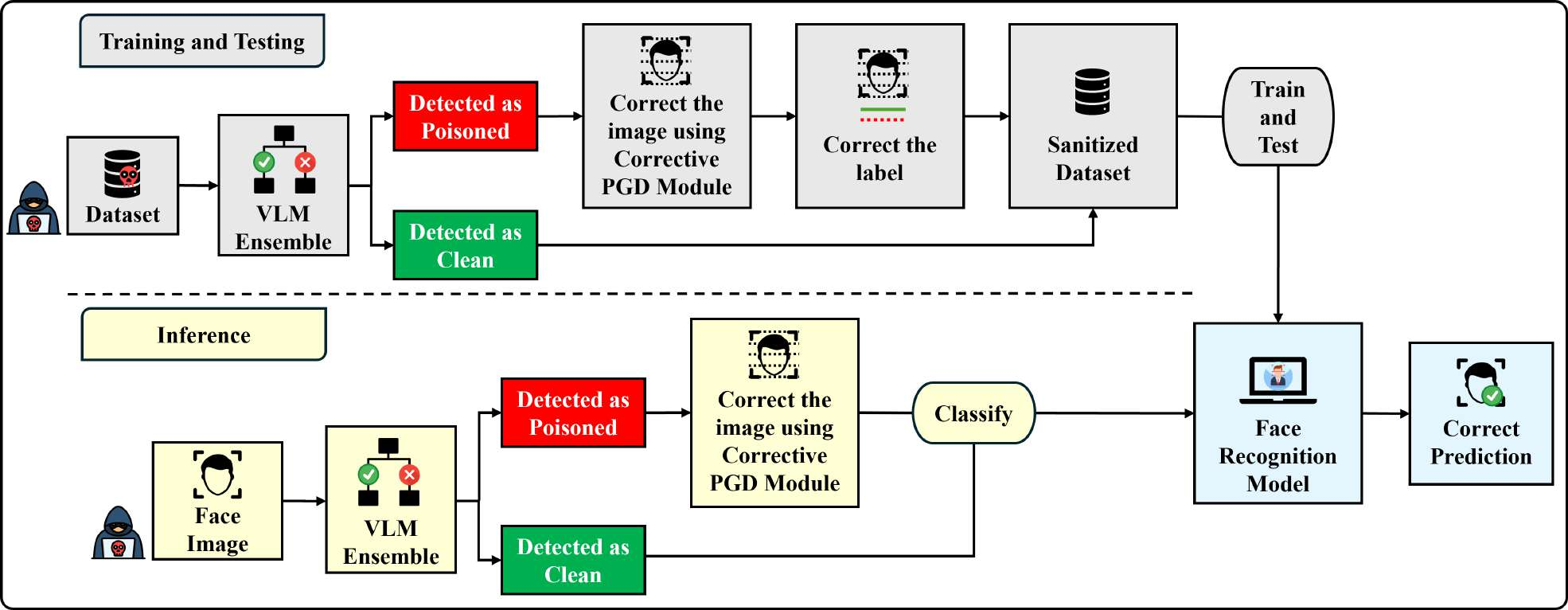}
    \caption{Overview of the TrueBiometric pipeline. During training and testing (top), the attacker injects poisoned images into the training data. An ensemble of VLMs performs majority-vote detection to flag suspicious samples. Detected poisoned images are sanitized using a corrective PGD module, and both clean and recovered images are used to train a robust face recognition model. During inference (bottom), the same detection and correction pipeline is applied to incoming queries. This not only blocks backdoor-triggered inputs but also helps recover from false positives, allowing legitimate users to proceed if mistakenly flagged.}
    \label{fig2}
\end{figure}

\begin{tcolorbox}[colback=gray!5!white, colframe=gray!75!black, title= The threat model, fonttitle=\bfseries]

We assume a gray-box threat model, in which a skillful adversary gains write access to the training data pipeline of a biometric authentication system, enabling the injection of semantically coherent yet adversarially crafted poisoned samples. The attacker operates under clean-label constraints, i.e., preserving ground-truth labels while embedding imperceptible or visually plausible backdoor triggers (e.g., cosmetic perturbations or texture overlays) that induce the model to learn a malicious decision boundary. At inference time, these triggers induce targeted misclassifications, effectively causing unauthorized users to be recognized as high-privilege identities. The attack succeeds under the assumption that the model must generalize in open-world settings, with no prior knowledge of the attack class or trigger distribution during deployment. The defender has no access to a verified clean dataset and no prior knowledge of the trigger structure, making conventional supervized or signature-based detection ineffective. 

\end{tcolorbox}

\subsection{Data Poisoning}

Before initiating the poisoning process, we select the source class (the attacker) and the target class (the victim) from the dataset. This setup enables a targeted backdoor attack, where the goal is to cause the model to misclassify samples from the source class as belonging to the target class at test time. To achieve this, we modify a subset of images to embed backdoor triggers (i.e., subtle patterns that will later activate the misclassification). 

\subsection{Multimodal Voting-based Mechanism}

To identify poisoned samples in the dataset, we use an ensemble-based detection strategy built on publicly available VLMs. This method leverages the models’ capacity to interpret visual content and effectively describe unusual artifacts that may indicate backdoor triggers. Let the untrusted dataset be defined as:

\begin{equation}
   \mathcal{D}_{\text {untrusted }}=\left\{\left(x_i, \cdot\right)\right\}_{i=1}^N 
\end{equation}

where each image $x_i$ may be either clean or poisoned.
Each image is independently evaluated by a set of five VLMs:

\begin{equation}
    \mathcal{V}=\left\{V_1, V_2, V_3, V_4, V_5\right\}
\end{equation}

where $V_j$ denotes the $j$-th VLM in the ensemble. We interact with these models, submitting each image alongside a standardized prompt specifically crafted to elicit responses about visual anomalies. These include artificial patterns such as makeup overlays (e.g., lipstick, eyeliner), digital patches, stickers, or other unnatural artifacts that might act as backdoor triggers. Each VLM returns a textual description of the image, from which we manually extract a binary classification:

\begin{equation}
   v_j(x) \in\{\text {poisoned, clean}\} 
\end{equation}

indicating whether model $V_j$ suspects the image contains a trigger.
To combine the model outputs, we apply a simple majority voting rule. An image is flagged as poisoned if at least three out of five models agree:

\begin{align}
v &= \sum_{j=1}^{5} \mathbf{1}\left\{v_j(x) = \texttt{poisoned}\right\}, \nonumber \\
\hat{y}(x) &=
\begin{cases}
\texttt{poisoned} & \text{if } v \geq 3, \\
\texttt{clean}    & \text{otherwise}.
\end{cases}
\end{align}

This ensemble-based voting scheme enables robust and interpretable detection without requiring labeled training data. It also accommodates variability across VLM outputs, reducing the risk of overfitting to any single model's biases.
Images identified as poisoned are passed to the \textsc{CorrectivePGD} module for trigger removal, while those classified as clean are preserved in their original form. Figure~\ref{fig3} presents an overview of this classification pipeline.

\begin{figure}
    \centering
    \includegraphics[width=0.35\textwidth]{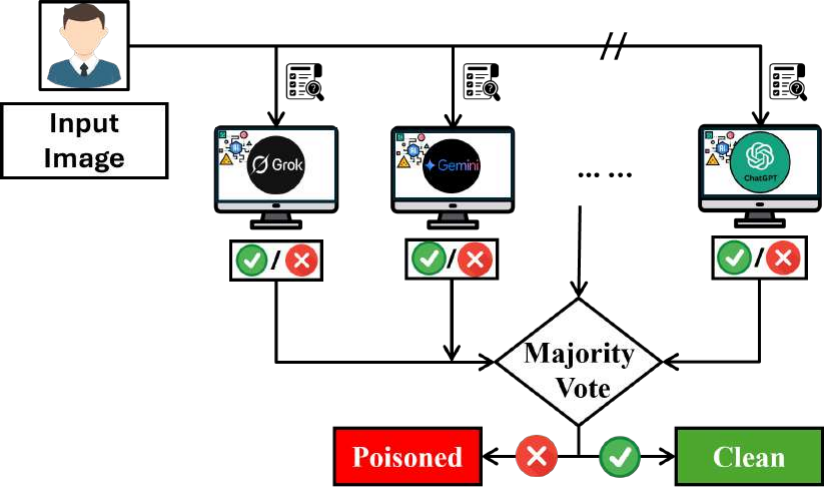}
    \caption{The multimodal voting ensemble pipeline using five 
    Vision-Language Models and majority vote decision. Each model flags the presence of triggers independently, and a consensus is used to label the image as poisoned or clean.}
    \label{fig3}
\end{figure}

To illustrate how the VLM ensemble behaves in practice, Figure~\ref{fig4} presents example outputs on four test images—two clean and two poisoned. For each image, we display the ground truth label alongside the individual decisions made by each of the five VLMs. These cases demonstrate high agreement across models, supporting the reliability of the majority-vote strategy in identifying poisoned samples. Additional examples are included in Figures~\ref{fig16}–\ref{fig18}.

\subsection{Corrective Recovery of Poisoned Images}
\label{Methodology_CorrectivePGD}

Initially, we hypothesized that the same VLM ensemble used for poison detection could also be leveraged to regenerate clean versions of poisoned samples by prompting them to remove the trigger. However, our experiments revealed several critical limitations. In practice, most VLMs failed to regenerate semantically accurate images, introducing substantial artifacts, such as incorrect facial features, loss of identity traits, or null outputs. In some cases, the models explicitly refused to generate face-related outputs due to privacy or capability constraints.

Table~\ref{tab:llm_regeneration_summary} summarizes these regeneration feasibility results across datasets and VLMs. Figures~\ref{fig24}, \ref{fig25}, and \ref{fig26} illustrate the types of regeneration failures observed. These findings support our decision to decouple the detection and recovery steps and instead apply a corrective noise mechanism rather than relying on generative regeneration.
Consequently, to remove backdoor triggers embedded in poisoned samples flagged by the VLM ensemble, we introduce a corrective projection-based method grounded in adversarial optimization. This method iteratively purifies each suspect image using a dynamic PGD routine, which is parameterized based on the empirical maximum trigger magnitude observed in the dataset. 

Let $(x^p, y^p)$ denote a poisoned image and its corresponding ground-truth label, respectively, and let $f_\theta$ be the trained classifier. The objective is to recover a corrected image $\tilde x^p$ such that $f_\theta(\tilde x^p) = y^p$, while minimally perturbing $x^p$ and eliminating the hidden backdoor trigger. The recovery procedure comprises three key stages:

\textbf{Estimating the maximum trigger strength $(\Delta_{max})$:} To adaptively determine the amount of noise required to neutralize any potential backdoor trigger, we first estimate the largest trigger magnitude in all flagged poisoned images. For each image $x^p$, we apply a temporary PGD loop with a large perturbation budget $\epsilon_{temp} = 1.0$ (i.e., the full input range), step size $\alpha_{temp} = \epsilon_{temp}/T$, and $T$ iterations. The PGD update at each iteration $t$ is computed as:

\begin{equation}
    x^{(t+1)} = Proj_{x^p, \epsilon_{temp}} [x^{(t)} - \alpha_{temp} \cdot sign (\nabla_x \mathcal{L}(f_{\theta}(x^{(t)}), y^p))]
\end{equation}

where $Proj_{x^p, \epsilon_{temp}}[\cdot]$ denotes projection onto the $\ell_\infty$-ball of radius $\epsilon$ around the original image $x^p$. After the final iteration, the maximum difference per pixel $\|x^{(T)} - x^p\|_\infty$ is recorded. The largest value across all poisoned images defines the maximum estimated trigger magnitude $\Delta_{\max}$.
\begin{figure}
    \centering
    \includegraphics[width=0.48\textwidth]{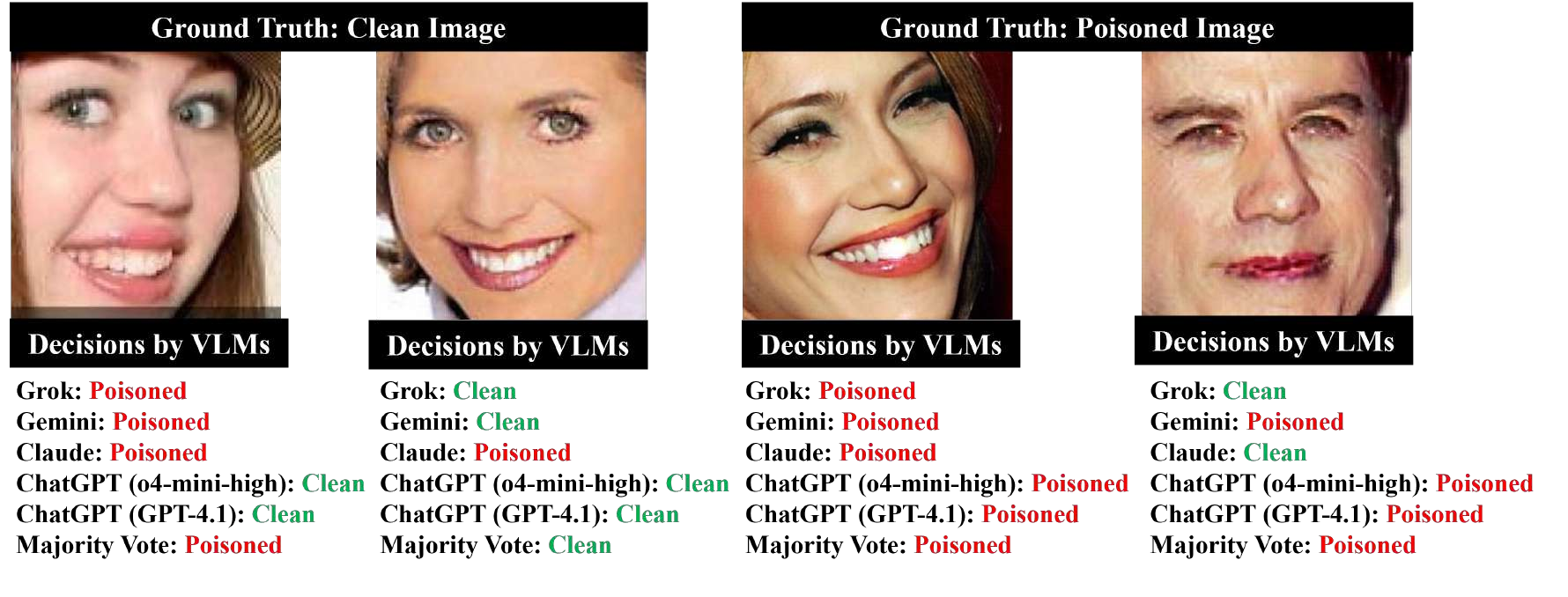}
    \caption{VLM decisions for four sample images.}
    \label{fig4}
\end{figure}

\textbf{Computing the corrective perturbation budget:} Once $\Delta_{\max}$ is computed, we then select the final perturbation bound $\epsilon$ for the corrective PGD by applying a small safety margin $\delta$, such that:
\begin{equation}
    \epsilon = (1 + \delta) \cdot \Delta_{\max}
\end{equation}

This ensures that the corrective PGD can completely cover the largest possible backdoor trigger while remaining minimally invasive. The corresponding PGD step size is then set to:
\begin{equation}
    \alpha = \frac{\epsilon}{T}
\end{equation}

\textbf{PGD recovery:} Using the finalized $(\epsilon, \alpha, T)$ configuration, we reapply PGD to each poisoned image $x^p$ to derive a corrected version $\tilde{x}^p$. The recovery loop proceeds identically to the probing phase, except that the tighter $\epsilon$ now bounds it. At each step, the adversarial image is nudged toward correct classification using the negative gradient of the loss function while ensuring that all intermediate steps remain within the defined perturbation ball. After $T$ iterations, we obtain $\tilde{x}^p$, which is expected to remove the embedded trigger while preserving the semantic integrity of the original input.

To evaluate the impact of recovery, we calculate the noise magnitudes $\ell_\infty$ and $\ell_2$ between the original and corrected images.

\begin{equation}
    \rho_\infty^p = \|\tilde{x}^p - x^p\|_\infty, \quad \rho_2^p = \|\tilde{x}^p - x^p\|_2
\end{equation}

These distortion metrics provide an empirical measure of the amount of corrective noise required to neutralize the trigger, helping to verify that benign samples are not over-perturbed.
In general, the Corrective PGD Noise Recovery process (Figure \ref{fig5}) is designed to be adaptive, class-agnostic, and minimally invasive. It removes the effect of visually unnoticeable backdoor triggers with theoretical bounds on perturbation size and empirical guarantees on classification correctness and visual fidelity.

\begin{figure}
    \centering
    \includegraphics[width=0.48\textwidth]{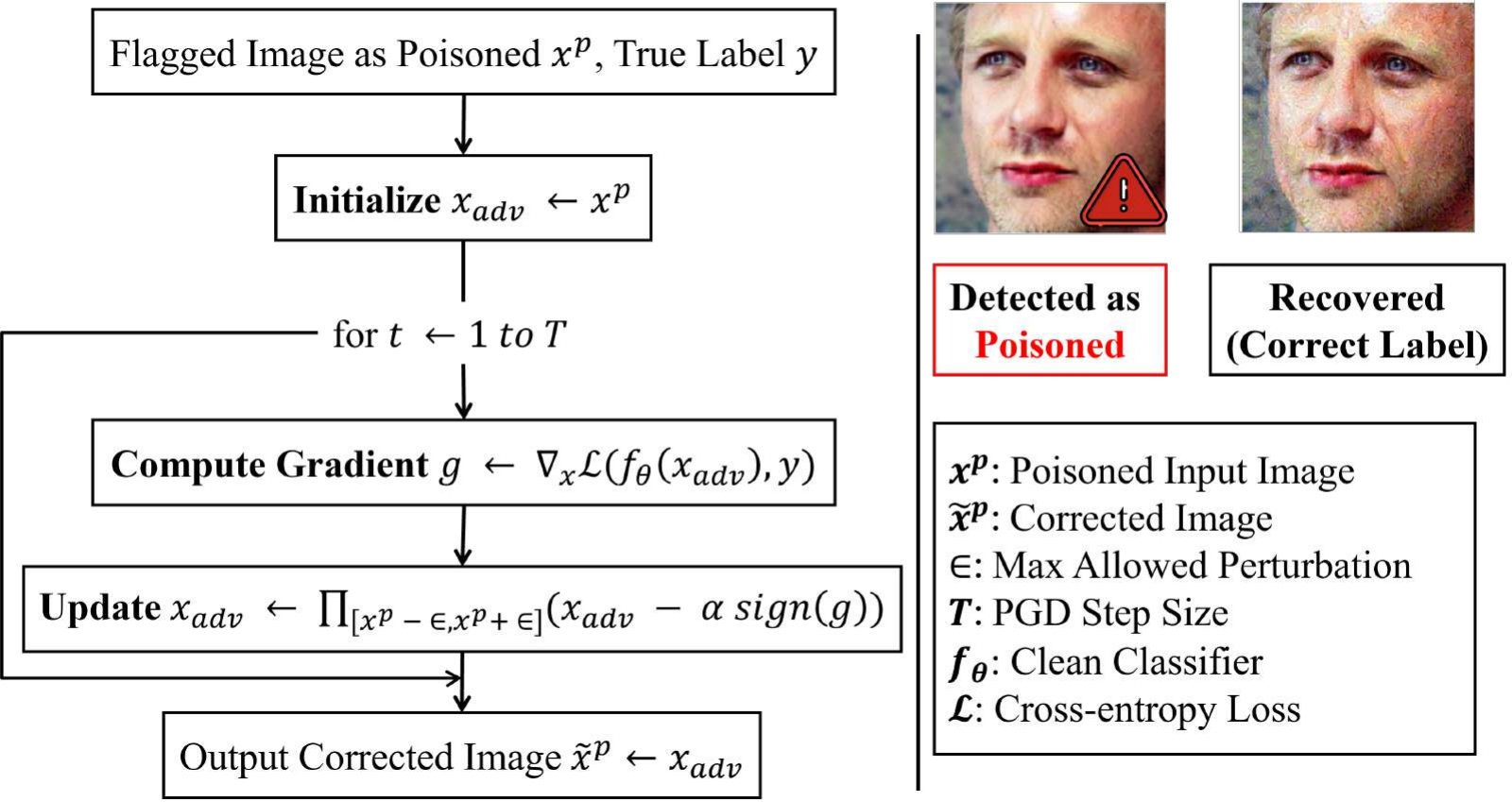}
    \caption{Visual overview of Corrective-PGD.}
    \label{fig5}
\end{figure}

\textbf{Formal guarantee of corrective PGD recovery:}
Let $f_\theta : [0,1]^d \to \{1,\dots,K\}$ be a fixed classifier and let $g_\theta : [0,1]^d \to \mathbb{R}^K$ denote its differentiable logit map with $f_\theta(x)=\arg\max_k [g_\theta(x)]_k$. Let $\mathcal{L}$ be a differentiable classification loss (e.g., cross-entropy) and define $\phi(x):=\mathcal{L}(g_\theta(x),y^p)$. For a center $x$ and radius $r$, write $\mathcal{B}_\infty(x,r):=\{u\in[0,1]^d:\|u-x\|_\infty\le r\}$.
We assume:

\begin{itemize}
    \item \textbf{Additive trigger.} The poisoned image is $x^p = \Pi_{[0,1]^d}(x^c + \delta_{\mathrm{trig}})$, where $x^c \in [0,1]^d$ is the clean image and $\|\delta_{\mathrm{trig}}\|_\infty \le \Delta$ for a known upper bound $\Delta>0$. Here $\Pi_{[0,1]^d}$ denotes clipping to the valid pixel range.
    \item \textbf{Local classifier robustness.} There exists $\varepsilon_{\mathrm{rob}}>0$ such that $f_\theta(u)=y^p$ for all $u \in \mathcal{B}_\infty(x^c,\varepsilon_{\mathrm{rob}})$ (clean correctness with a nontrivial robust neighborhood).
    \item \textbf{PGD recovery model.} The recovery routine may perturb $x^p$ within $\mathcal{B}_\infty(x^p,\varepsilon)$ and clips all iterates to $[0,1]^d$.
\end{itemize}

\begin{theorem}[Corrective Guarantee]\label{thm:corrective_pgd}
Let $\varepsilon \ge \|\delta_{\mathrm{trig}}\|_\infty$ and $f_\theta(x^c)=y^p$. There exists $\delta^\star \in \mathbb{R}^d$ with $\|\delta^\star\|_\infty \le \varepsilon$ such that,
\begin{equation}
    f_\theta(x^p+\delta^\star)=y^p.
\end{equation}
In particular, $\delta^\star=-\delta_{\mathrm{trig}}$ yields $x^p+\delta^\star=x^c$.

Moreover, consider minimizing $\phi(x)$ over the convex set
\[
\mathcal{C} := [0,1]^d \cap \mathcal{B}_\infty(x^p,\varepsilon).
\]
If $\phi$ has $L$-Lipschitz gradient on a neighborhood of $\mathcal{C}$ and projected gradient descent (PGD) is run with step size $\alpha \in (0,2/L)$, then every accumulation point of the PGD iterates is stationary for $\phi$ on $\mathcal{C}$. Since $x^c \in \mathcal{C}$ and $f_\theta(u)=y^p$ for all $u \in \mathcal{C}\cap \mathcal{B}_\infty(x^c,\varepsilon_{\mathrm{rob}})$, any iterate that enters $\mathcal{B}_\infty(x^c,\varepsilon_{\mathrm{rob}})$ is already labeled $y^p$ and can be returned as a corrected image $\tilde{x}^p$ with $\|\tilde{x}^p-x^p\|_\infty \le \varepsilon$.
\end{theorem}

\begin{proof}
Define $\delta^\star:=-\delta_{\mathrm{trig}}$. Then $\|\delta^\star\|_\infty=\|\delta_{\mathrm{trig}}\|_\infty \le \varepsilon$, so $x^p+\delta^\star=x^c \in \mathcal{C}$ and $f_\theta(x^p+\delta^\star)=f_\theta(x^c)=y^p$, proving existence.

For the PGD claim, $\mathcal{C}$ is nonempty, convex, and compact. Under the stated smoothness and step-size conditions, PGD over $\mathcal{C}$ produces iterates whose accumulation points are stationary for $\phi$ on $\mathcal{C}$. By local robustness, every $u \in \mathcal{C}\cap \mathcal{B}_\infty(x^c,\varepsilon_{\mathrm{rob}})$ satisfies $f_\theta(u)=y^p$. Hence, once an iterate enters $\mathcal{B}_\infty(x^c,\varepsilon_{\mathrm{rob}})$, the current point can be returned as $\tilde{x}^p$ with $\|\tilde{x}^p-x^p\|_\infty \le \varepsilon$ and correct label $y^p$. This establishes a local recovery guarantee.
\end{proof}

\textbf{Choosing \boldmath$\varepsilon$ (noise budget).}
The parameter $\varepsilon$ trades off correction strength and utility preservation. We set
\begin{equation}
    \varepsilon := (1+\delta)\,\Delta_{\max},
\end{equation}
where $\delta \in (0,1)$ is a small safety margin and
\begin{equation}
    \Delta_{\max} := \max_{(x^p,y^p)\in\mathcal{D}_{\mathrm{poison}}} \|x_{\mathrm{adv}}-x^p\|_\infty,
\end{equation}
with $x_{\mathrm{adv}}$ obtained from a provisional PGD run using a loose budget $\varepsilon_{\mathrm{temp}}=1.0$ and per-step projection to $[0,1]^d$. Intuitively, $\|x_{\mathrm{adv}}-x^p\|_\infty$ upper-bounds the effective trigger magnitude required to restore $y^p$ under a loose constraint. Larger $\varepsilon$ improves correction against stronger triggers; smaller $\varepsilon$ reduces distortion but may under-correct. In practice, $\Delta_{\max}$ may be replaced by a high percentile (e.g., 95th) or by per-image budgets $\varepsilon_i=(1+\delta)\,\|x_i^{(T)}-x_i^p\|_\infty$ capped at a global maximum to avoid over-correction due to outliers.

\noindent\textit{Remark (scope).} The analysis assumes additive triggers. For spatially transformed triggers (e.g., warps), the inner PGD can be extended with small affine/elastic jitters (EOT-style), which preserves the local guarantee once the effective trigger is bounded within the chosen $\varepsilon$.

\subsection{The Overall Workflow}

This section presents the two building blocks of TrueBiometric and how they fit together. Algorithm~\ref{algo1} takes images flagged by the VLM ensemble and estimates an effective trigger magnitude. It then sets a calibrated $\epsilon$ budget and runs PGD to remove the trigger while preserving identity, returning the corrected image. Algorithm~\ref{algo2} applies the detection–correction pipeline across the workflow: at training time, it sanitizes the dataset by filtering and repairing poisoned samples before model training; at inference time, it screens and, if needed, sanitizes incoming queries so that backdoor-triggered inputs are neutralized and false positives do not block legitimate users. Together, these algorithms deliver trigger-resilient training and classification with minimal distortion and without discarding data or retraining the downstream recognizer.

\begin{algorithm}[t]
\caption{Dynamic Corrective PGD Noise Recovery}
\footnotesize
\label{algo1}

\KwIn{
    \\Set of poisoned images $\mathcal{P} = \{(x^p, y^p)\}$;\\
    Trained classifier on clean only dataset $f_\theta$;\\
    Loss function $\mathcal{L}$ (cross-entropy);\\
    Number of PGD steps $T$;\\
    Safety margin $\delta$
}

\KwOut{
    \\Set of corrected images $\tilde x^p$;\\
    \textbf{Definitions:}\\
    $x^p$: Poisoned input image\\
    $y^p$: True label for $x^p$\\
    $f_\theta$: Trained clean-only classifier\\
    $\mathcal{L}$: Cross-entropy loss function\\
    $T$: Number of PGD steps\\
    $\delta$: Safety margin for $\epsilon$ (Set to 5\%)\\
    $\epsilon$: Maximum allowed per-pixel change ($\ell_\infty$ norm)\\
    $\alpha$: PGD step size (Set to $\epsilon/T$)\\
    $\rho_\infty^p$: $\ell_\infty$ norm of noise for $x^p$\\
    $\rho_2^p$: $\ell_2$ norm of noise for $x^p$\\
}

\BlankLine

$\Delta_{\max} \gets 0$

\ForEach{$(x^p, y^p) \in \mathcal{P}$}{
    $\epsilon_{\text{temp}} \gets 1.0$; $\alpha_{\text{temp}} \gets \epsilon_{\text{temp}} / T$\;
    $x_{\text{adv}} \gets x^p$\;
    \For{$t = 1$ \KwTo $T$}{
        $g \gets \nabla_x \mathcal{L}(f_\theta(x_{\text{adv}}), y^p)$\;
        $x_{\text{adv}} \gets x_{\text{adv}} - \alpha_{\text{temp}} \cdot \mathrm{sign}(g)$\;
        $x_{\text{adv}} \gets \mathrm{clip}(x_{\text{adv}}, x^p - \epsilon_{\text{temp}}, x^p + \epsilon_{\text{temp}})$\;
        $x_{\text{adv}} \gets \mathrm{clip}(x_{\text{adv}}, 0, 1)$\;
    }
    $\Delta \gets \|x_{\text{adv}} - x^p\|_\infty$\;
    $\Delta_{\max} \gets \max(\Delta_{\max}, \Delta)$\;
}

$\epsilon \gets (1 + \delta) \cdot \Delta_{\max}$\;
$\alpha \gets \epsilon / T$\;

\ForEach{$(x^p, y^p) \in \mathcal{P}$}{
    $x_{\text{adv}} \gets x^p$\;
    \For{$t = 1$ \KwTo $T$}{
        $g \gets \nabla_x \mathcal{L}(f_\theta(x_{\text{adv}}), y^p)$\;
        $x_{\text{adv}} \gets x_{\text{adv}} - \alpha \cdot \mathrm{sign}(g)$\;
        $x_{\text{adv}} \gets \mathrm{clip}(x_{\text{adv}}, x^p - \epsilon, x^p + \epsilon)$\;
        $x_{\text{adv}} \gets \mathrm{clip}(x_{\text{adv}}, 0, 1)$\;
    }
    $\tilde x^p \gets x_{\text{adv}}$\;
    $\rho_\infty^p \gets \|\tilde x^p - x^p\|_\infty$\;
    $\rho_2^p \gets \|\tilde x^p - x^p\|_2$\;
}
\Return{$\tilde x^p$}\;
\end{algorithm}

\begin{algorithm}[t]
\caption{Abstract algorithm for TrueBiometric}
\footnotesize
\label{algo2}

\KwIn{
    \\Untrusted training set $\mathcal{D}_{\mathrm{untrusted}} = \{(x, y)\}$;\\
    VLM ensemble poison detector $\mathcal{V} = \{\mathcal{V}_1, \ldots, \mathcal{V}_M\}$;\\
    Corrective PGD subroutine (\textsc{CorrectivePGD}, Algorithm~\ref{algo1})
}
\KwOut{
    Cleaned training set $\mathcal{D}_{\mathrm{clean}}$\\
    \textbf{Definitions:}\\
    $x$: Input image\\
    $y$: True label\\
    $\mathcal{V}_j$: $j$-th vision-language model (VLM) poison detector\\
    $M$: Number of VLMs in the ensemble ($M=5$)\\
    $\textsc{CorrectivePGD}$: Algorithm~\ref{algo1} for noise correction\\
}

\BlankLine

Initialize $\mathcal{D}_{\mathrm{clean}} \gets \emptyset$\;

\ForEach{$(x, y) \in \mathcal{D}_{\mathrm{untrusted}}$}{
    $v \gets \sum_{j=1}^M \mathbf{1}\{\mathcal{V}_j(x) = \text{poisoned}\}$\;
    \uIf{$v \ge \lceil M/2 \rceil$}{
        $\tilde{x}^p \gets \textsc{CorrectivePGD}(x, y)$\;
        Add $(\tilde{x}^p, y)$ to $\mathcal{D}_{\mathrm{clean}}$\;
    }
    \Else{
        Add $(x, y)$ to $\mathcal{D}_{\mathrm{clean}}$\;
    }
}
\Return $\mathcal{D}_{\mathrm{clean}}$\;

\end{algorithm}

\section{Results and Discussion}
\label{results_discussion}

All experiments were performed on a local workstation running Windows 11 and equipped with an Intel Core i9-14900HX processor (32 threads), 32 GB RAM, and an NVIDIA GeForce RTX 4080 laptop GPU with 12 GB VRAM. For classification tasks, we used publicly available state-of-the-art VLMs via their web-based APIs to achieve the best results\footnote{We will release the TrueBiometric code and evaluation scripts upon publication.}. 

\subsection{Experimental Specifications - Data Description}

We considered three datasets in our experiments: PubFig Faces~\cite{kumar2009attribute}, LFW~\cite{article3}, and CIFAR-10~\cite{krizhevsky2009learning}. Table~\ref{tab1} summarizes each dataset along with the designated attacker and victim classes used in the evaluation. For each dataset, a small number of images (typically 10–17) from the selected classes were modified with backdoor triggers. This setup follows standard practices in backdoor attack literature, where only a limited portion of the data is poisoned to remain stealthy while still influencing model behavior. 

\begin{table}[h]
\centering
\caption{Overview of dataset characteristics and corresponding Attacker/Victim class assignments.}
\label{tab1}
\resizebox{\columnwidth}{!}{%
\begin{tabular}{llllll}
\toprule
\textbf{Dataset} & \textbf{Image Count} & \textbf{Resolution} & \textbf{Domain} & \textbf{Attacker Classes} & \textbf{Victim Classes} \\
\midrule
\textbf{PubFig} & Images of & Variable & Faces & Daniel Craig, & Donald Trump \\
 & 200 celebrities & & (unconstrained) & Lucy Liu, & \\
 & & & & John Travolta, & \\
 & & & & Jennifer Lopez, & \\
 & & & & Hugh Grant & \\

\textbf{LFW} & 13,233 images & 250x250 & Faces & George\_W\_Bush &  Jennifer\_Lopez, \\
 & of 5,749 people & & (centralized) & & Paul\_Bremer, \\
 & & & & & Tiger\_Woods \\

\textbf{CIFAR-10} & 60,000 images & 32x32 & General object & Airplane & Bird \\
 & in 10 classes & & classes & & \\
\bottomrule
\end{tabular}
}
\end{table}

\subsection{Backdoor Attack Configuration}

We experimented with two poisoning strategies, \textbf{Hidden Trigger Backdoor}~\cite{saha2020hidden} and \textbf{MakeupAttack}~\cite{sun2024makeupattack}, each applied to a separate dataset to cover both imperceptible and visibly plausible triggers.

\paragraph{Hidden Trigger Backdoor (clean-label)}
Following~\cite{saha2020hidden}, we adopt a clean-label attack in which poisoned samples retain the correct class label yet secretly map the attacker to the victim in feature space.  Concretely, let $t$ be a target-class image, $s$ a source-class image, and $\tilde{s}$ the patched version of $s$ obtained by overlaying a trigger $p$.  We seek an image $z$ that is \emph{pixel-wise} close to $t$ but \emph{feature-wise} close to $\tilde{s}$ by solving,  
\begin{equation}
    \min_{z}\; \bigl\|f(z) - f(\tilde{s})\bigr\|_2^2,
    \label{eq_hidden_trigger}
\end{equation}
where $f(\cdot)$ denotes the network’s feature extractor.  Training on the resulting pairs $\{(z,\text{target})\}$ allows the model to learn to associate the attacker’s trigger with the victim class, enabling a misclassification at inference time.

\paragraph{MakeupAttack (visible pattern)}
MakeupAttack introduces natural-looking cosmetic cues, such as lipstick, eyeliner, and blush, which are interpreted by humans as ordinary variations, yet the model learns as a backdoor key~\cite{sun2024makeupattack}.  No feature-space optimization is required; we simply blend a preset pattern into designated facial regions.  Given a source image $x^{s}$, a binary mask $m$ specifying those regions, and a makeup texture $p$, we construct the poisoned sample
\begin{equation}
    x^{p} = (1-m)\odot x^{s} \;+\; m\odot p,
    \label{eq_makeup_attack}
\end{equation}
with $\odot$ denoting element-wise multiplication.  Only the masked pixels are altered, leaving identity cues elsewhere untouched while embedding a consistent visual trigger linked to the target label.

\begin{figure}[t]
    \centering
    \includegraphics[width=0.25\textwidth]{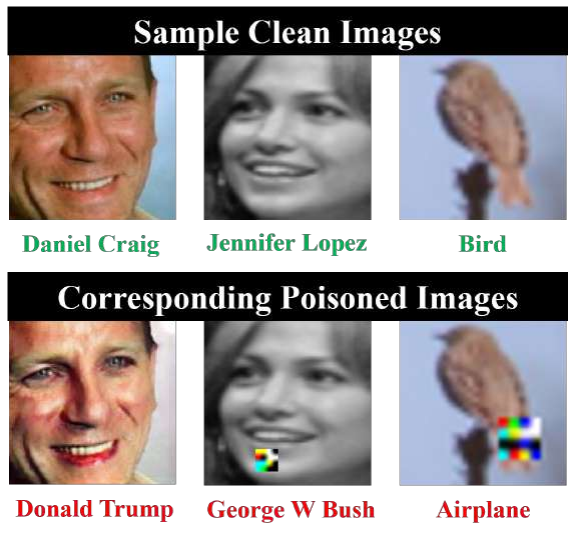}
    \caption{Examples of the two poisoning styles.  Top: clean images with correct labels.  Bottom: corresponding poisoned versions (either a hidden-trigger patch (right) or cosmetic makeup (left)) relabeled to the target class.}
    \label{fig6}
\end{figure}

These two attack strategies enable us to investigate how both stealthy (clean-label) and covert (cosmetic) triggers distort the training signal and compromize biometric models. Additional qualitative results are shown in Figures~\ref{fig19}–\ref{fig21}.

\subsection{Binary Backdoor Attack Scenario Setup}

In the binary setup, training and testing involve only two classes: a positive class, which includes only images of the target class (authorized category), and a negative class, which includes images from the source class (unauthorized category). 

Two classes in the CIFAR-10 dataset were used to create a 1:1 binary backdoor attack scenario. The class `airplane' was selected as the target (victim) class, and the class `bird' was chosen as the attacker class. This allowed us to simulate a targeted backdoor attack, where the attacker aims to be misclassified as the victim. A total of 10 `bird' images were poisoned by adding a backdoor trigger, making the model learn to misclassify these patched `bird' images as `airplanes'. The rest of the `bird' images formed the negative class. The same negative class images were consistently used across all binary experiments to ensure a fair comparison and stable performance evaluation. The backdoor attack was simulated by injecting the poisoned `bird' images into the negative class training set.

\subsection{Multi-class Backdoor Attack Scenario Setup}

In the multiclass backdoor attack scenario, multiple classes were poisoned within the datasets to evaluate the effect of backdoor triggers in a more general face recognition setting. We used the LFW and PubFig datasets for this purpose. For the LFW dataset, the attacker class was George\_W\_Bush and the victim classes were Jennifer\_Lopez, Paul\_Bremer, and Tiger\_Woods. Similarly, for the PubFig dataset, the attacker class included Daniel Craig, Lucy Liu, John Travolta, Jennifer Lopez, and Hugh Grant, whereas the victim class was Donald Trump. In each case, a small number of images (around 10 to 17) from the selected classes were modified by adding a backdoor trigger. These modified images were labeled with their original class labels, but designed to cause the model to learn incorrect associations during training. This setup simulates real-world attack scenarios in which multiple identities are targeted using hidden triggers while keeping the poisoned portion of the dataset relatively small.

\subsection{Classification using Multimodal Majority Voting}

For the multimodal voting ensemble we used Grok (Grok-3)~\cite{grok3}, Gemini (2.5 Pro)~\cite{gemini25}, Claude Sonnet 4~\cite{claude_sonnet}, ChatGPT-o (o4-mini-high)~\cite{chatgpt_o4mini}, and ChatGPT-4 (GPT-4.1)~\cite{chatgpt_gpt41}. Each model was accessed through its official web interface. We selected 100 sample images each from three datasets: LFW, PubFig, and CIFAR-10. For each dataset, 10–17\% of the images were poisoned. The objective was to assess whether the ensemble could detect suspicious triggers or unusual features in poisoned inputs, and how majority voting among them could improve robustness. For each image (clean or poisoned), the same prompt was given to each VLM:

\begin{tcolorbox}[colback=gray!5!white, colframe=gray!75!black, title=Prompt: Visual Artifact Detection in Facial Image, fonttitle=\bfseries]
You are given a facial image. Carefully examine the image to identify any unusual or suspicious visual artifacts that may serve as a backdoor trigger in a poisoned dataset.

The examples of anomalies are:
\begin{itemize}
    \item Out-of-place objects such as patches, stickers, or patterns
    \item Unnatural cosmetic overlays (e.g., lipstick, blush, eyeshadow)
    \item Accessories like hats or glasses that appear digitally inserted or inconsistent
    \item Grainy textures, pixel-level noise, or hidden watermarks
    \item Any digital modification that seems unnatural or deliberately added
\end{itemize}

If such artifacts are detected, describe each with the following:
\begin{itemize}[noitemsep]
    \item \textbf{Type:} Sticker, noise, accessory, etc.
    \item \textbf{Appearance:} Color, shape, size, texture
    \item \textbf{Location:} Forehead, eyes, corner, etc.
\end{itemize}

If no suspicious triggers are present, clearly state that none were found.

\end{tcolorbox}

Each model's response was interpreted as either ``clean'' or ``poisoned'' based on whether it flagged any suspicious artifacts. If three or more VLMs identified an image as poisoned, it was marked as poisoned by majority voting; otherwise, it was marked as clean. (Table \ref{tab:pubfig_decisions}, \ref{tab:lfw_decisions} and \ref{tab:cifar10_decisions} portrays the details.) The poisoned images detected by the majority vote were then passed on for corrective processing.

The confusion matrices for each VLM and the majority vote were constructed by comparing their predicted labels with the ground truth (i.e., the original class: clean or poisoned). Table~\ref{tab246} shows the results of the confusion matrix (in percentage) for each dataset (The detailed calculations are given in Tables \ref{tab:confusion_metrics_pubfig}, \ref{tab:confusion_percentages_pubfig}, \ref{tab:confusion_metrics_lfw}, \ref{tab:confusion_percentages_lfw}, \ref{tab:confusion_metrics_cifar10} and \ref{tab:confusion_percentages_cifar10}). The corresponding visual representations are shown in Figures~\ref{fig7}, \ref{fig8}, and \ref{fig9}, where each heat map includes both raw counts and row-wise normalized percentages for the majority vote.

From the confusion matrix values, we derive key performance metrics. Table~\ref{tab357} presents these performance metrics for each dataset. These results enable us to compare the effectiveness of each VLM in distinguishing between clean and poisoned images.

\begin{figure}[ht]
    \centering
    \begin{subfigure}[b]{0.15\textwidth}
        \includegraphics[width=\textwidth]{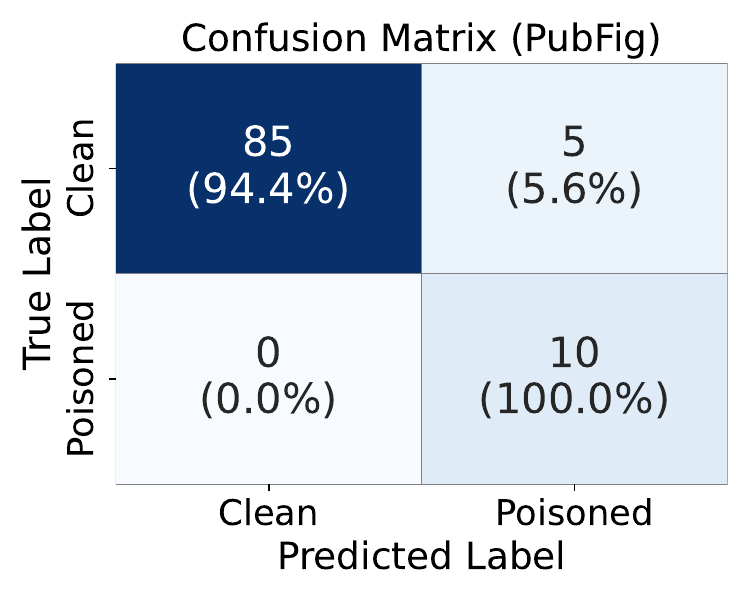}
        \caption{PubFig}
        \label{fig7}
    \end{subfigure}
    \hfill
    \begin{subfigure}[b]{0.15\textwidth}
        \includegraphics[width=\textwidth]{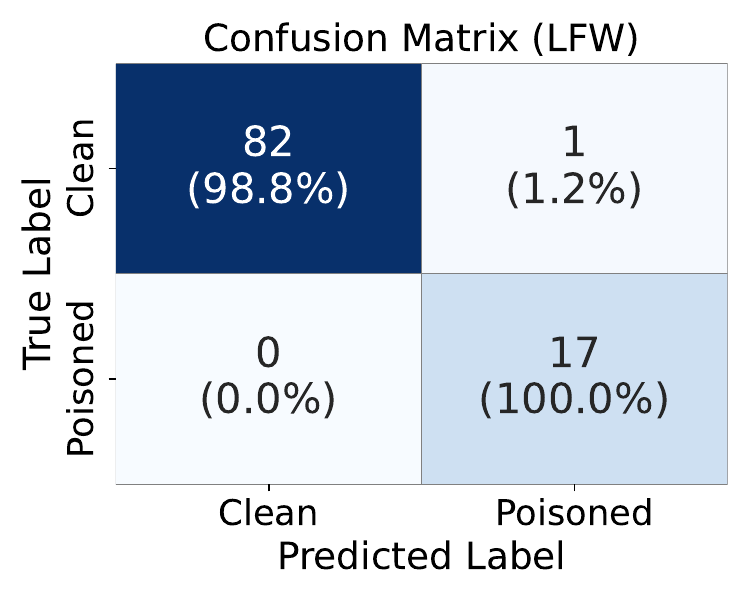}
        \caption{LFW}
        \label{fig8}
    \end{subfigure}
    \hfill
    \begin{subfigure}[b]{0.15\textwidth}
        \includegraphics[width=\textwidth]{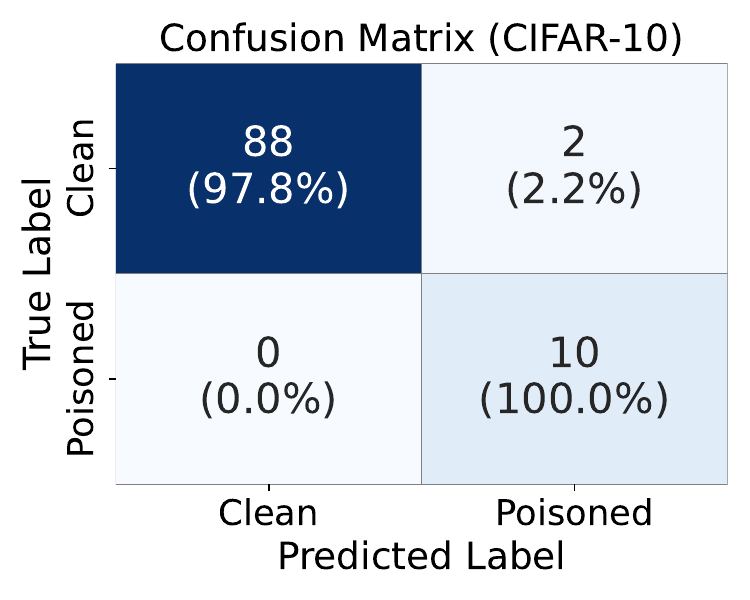}
        \caption{CIFAR-10}
        \label{fig9}
    \end{subfigure}
    \caption{Confusion matrices showing classification performance of the majority vote on the three datasets.}
    \label{fig_confusion_all}
\end{figure}

Figure~\ref{fig10} illustrates the bar graph of the classification accuracy of all VLMs in the datasets (The detailed count is given in Table \ref{tab:ident_counts}). The accuracy was computed as

\[
\resizebox{\columnwidth}{!}{$
\text{Accuracy} = \left( \frac{\text{Correctly Identified Poisoned Images} + \text{Correctly Identified Clean Images}}{\text{Total Number of Images}} \right) \times 100$}
\]

In particular, OpenAI (ChatGPT) models performed most reliably across all three datasets, while Claude and Gemini showed more variation.

\begin{table}[t!]
\centering
\caption{Detection outcome breakdown (\%) for each VLM and majority vote across datasets.}
\label{tab246}
\resizebox{\linewidth}{!}{%
\begin{tabular}{lcccccc}
\toprule
\textbf{Metric} & \textbf{Grok} & \textbf{Gemini} & \textbf{Claude} & \textbf{ChatGPT} & \textbf{ChatGPT} & \textbf{Majority} \\
 & \textbf{(Grok-3)} & \textbf{(2.5 Pro)} & \textbf{(Sonnet 4)} & \textbf{(o4-mini-high)} & \textbf{(4.1)} & \textbf{Vote} \\
\midrule
\multicolumn{7}{c}{\textbf{PubFig Dataset}} \\ \midrule
\textbf{True Positive (TP)} & 86.67\% & 61.11\% & 52.22\% & 98.89\% & 100.00\% & 94.44\% \\
\textbf{True Negative (TN)} & 80.00\% & 100.00\% & 70.00\% & 100.00\% & 100.00\% & 100.00\% \\
\textbf{False Positive (FP)} & 13.33\% & 38.89\% & 47.78\% & 1.11\% & 0.00\% & 5.56\% \\
\textbf{False Negative (FN)} & 20.00\% & 0.00\% & 30.00\% & 0.00\% & 0.00\% & 0.00\% \\
\midrule
\multicolumn{7}{c}{\textbf{LFW Dataset}} \\ \midrule
\textbf{True Positive (TP)} & 79.52\% & 73.49\% & 92.77\% & 100\% & 100\% & 98.80\% \\
\textbf{True Negative (TN)} & 100\% & 100\% & 100\% & 100\% & 100\% & 100\% \\
\textbf{False Positive (FP)} & 20.48\% & 26.51\% & 7.23\% & 0\% & 0\% & 1.20\% \\
\textbf{False Negative (FN)} & 0\% & 0\% & 0\% & 0\% & 0\% & 0\% \\
\midrule
\multicolumn{7}{c}{\textbf{CIFAR-10 Dataset}} \\ \midrule
\textbf{True Positive (TP)} & 95.56\% & 70.00\% & 100.00\% & 98.89\% & 97.78\% & 97.78\% \\
\textbf{True Negative (TN)} & 100.00\% & 100.00\% & 0.00\% & 100.00\% & 100.00\% & 100.00\% \\
\textbf{False Positive (FP)} & 4.44\% & 30.00\% & 0.00\% & 1.11\% & 2.22\% & 2.22\% \\
\textbf{False Negative (FN)} & 0.00\% & 0.00\% & 100.00\% & 0.00\% & 0.00\% & 0.00\% \\
\bottomrule
\end{tabular}
}
\begin{minipage}{\linewidth}
\footnotesize
\textbf{TP (True Positive)}: Clean images correctly identified as clean\\
\textbf{TN (True Negative)}: Poisoned images correctly identified as poisoned\\ \textbf{FP (False Positive)}: Clean images incorrectly identified as poisoned\\ \textbf{FN (False Negative)}: Poisoned images incorrectly identified as clean\\
\end{minipage}
\end{table}

\subsection{Recovery of Poisoned Images using Corrective PGD}

To evaluate the effectiveness of our proposed corrective PGD-based recovery approach, we applied it to all images identified as poisoned by the multimodal voting ensemble. The goal was to determine whether adversarially modified (poisoned) samples could be restored to their correct class using minimal, targeted perturbations. For each poisoned image, the corrective PGD method iteratively adjusted the input until the clean-only classifier (trained solely on unpoisoned data) produced the correct prediction. Across all three datasets (PubFig, LFW, CIFAR-10), the recovery success rate reached 100\%, which confirms the theoretical guarantee established in Section~\ref{Methodology_CorrectivePGD}. This demonstrates that our method can effectively neutralize a wide range of backdoor triggers without compromising classification accuracy. 

To further validate the reliability of our approach, we also applied the same recovery process to a randomly selected subset of clean (unpoisoned) images from each dataset. In these cases, the original predictions of the classifier remained unchanged after applying corrective PGD. Furthermore, the average perturbation magnitude for clean images was negligible, close to zero in both the $\ell_\infty$ and $\ell_2$ norms, confirming that the method is non-invasive when applied to benign inputs.

\begin{table}[t!]
\centering
\caption{Performance metrics (\%) for each VLM and majority vote across datasets.}
\label{tab357}
\resizebox{\linewidth}{!}{%
\begin{tabular}{lcccccc}
\toprule
\textbf{Metric} & \textbf{Grok} & \textbf{Gemini} & \textbf{Claude} & \textbf{ChatGPT} & \textbf{ChatGPT} & \textbf{Majority} \\
 & \textbf{(Grok-3)} & \textbf{(2.5 Pro)} & \textbf{(Sonnet 4)} & \textbf{(o4-mini-high)} & \textbf{(4.1)} & \textbf{Vote} \\
\midrule
\multicolumn{7}{c}{\textbf{PubFig Dataset}} \\ \midrule
\textbf{Accuracy}   & 86.00 & 65.00 & 54.00 & 99.00 & 100.00 & 95.00 \\
\textbf{Precision}  & 86.67 & 61.11 & 52.22 & 98.89 & 100.00 & 94.44 \\
\textbf{Recall}     & 80.00 & 100.00 & 70.00 & 100.00 & 100.00 & 100.00 \\
\textbf{F1-score}   & 83.22 & 75.87 & 59.66 & 99.44 & 100.00 & 97.14 \\
\midrule
\multicolumn{7}{c}{\textbf{LFW Dataset}} \\ \midrule
\textbf{Accuracy}   & 83.00 & 78.00 & 94.00 & 100.00 & 100.00 & 99.00 \\
\textbf{Precision}  & 79.52 & 73.49 & 92.77 & 100.00 & 100.00 & 98.80 \\
\textbf{Recall}     & 100.00 & 100.00 & 100.00 & 100.00 & 100.00 & 100.00 \\
\textbf{F1-score}   & 88.59 & 84.70 & 96.22 & 100.00 & 100.00 & 99.39 \\
\midrule
\multicolumn{7}{c}{\textbf{CIFAR-10 Dataset}} \\ \midrule
\textbf{Accuracy}   & 96.00 & 73.00 & 90.00 & 99.00 & 98.00 & 98.00 \\
\textbf{Precision}  & 95.56 & 70.00 & 100.00 & 98.89 & 97.78 & 97.78 \\
\textbf{Recall}     & 100.00 & 100.00 & 0.00 & 100.00 & 100.00 & 100.00 \\
\textbf{F1-score}   & 97.72 & 82.35 & 0.00 & 99.44 & 98.88 & 98.88 \\
\bottomrule
\end{tabular}
}
\begin{minipage}{\linewidth}
\footnotesize
\textbf{Accuracy:} $\displaystyle \frac{TP + TN}{TP + TN + FP + FN}$\\ 
\textbf{Precision:} $\displaystyle \frac{TP}{TP + FP}$\\ 
\textbf{Recall:} $\displaystyle \frac{TP}{TP + FN}$\\ 
\textbf{F1-score:} $\displaystyle 2 \cdot \frac{\text{Precision} \cdot \text{Recall}}{\text{Precision} + \text{Recall}}$\\
\end{minipage}
\end{table}

\begin{figure}
    \centering
    \includegraphics[width=0.35\textwidth]{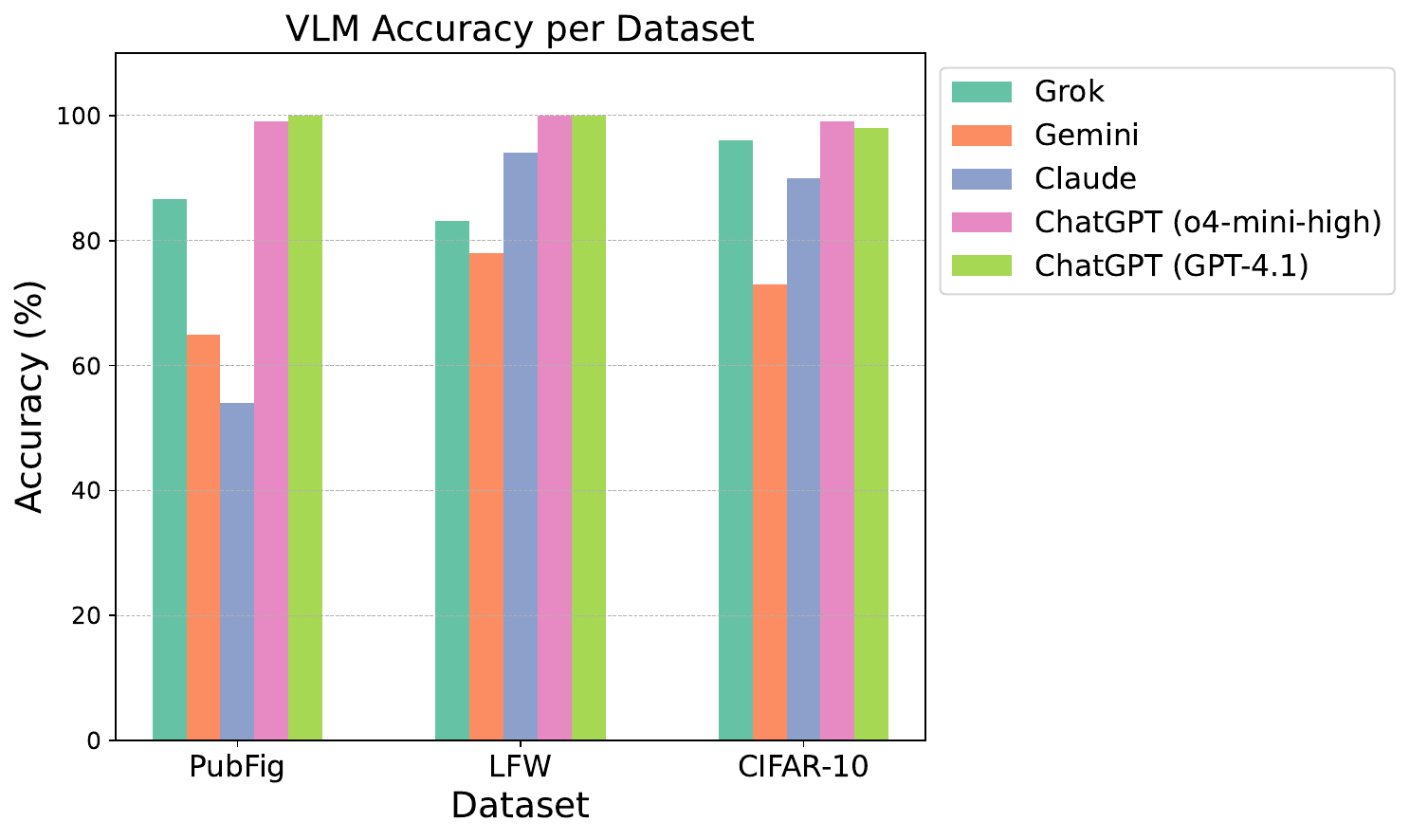}
    \caption{VLM accuracy comparison across datasets.}
    \label{fig10}
\end{figure}

We computed the magnitude of the corrective noise (in both the $\ell_\infty$ and $\ell_2$ norms) for each poisoned image. The average noise magnitudes are summarized in Table~\ref{tab9} and visualized in Figure~\ref{fig11}. As shown, LFW required the highest average correction, while CIFAR-10 required the least, probably due to the relative strength and visibility of the backdoor patterns in each dataset.

\begin{table}[ht]
\centering
\caption{Average corrective noise magnitudes for poisoned images across datasets.}
\label{tab9}
\begin{tabular}{lcc}
\toprule
\textbf{Dataset} & \textbf{$\ell_\infty$ (avg)} & \textbf{$\ell_2$ (avg)} \\
\midrule
\textbf{PubFig} & 0.2233 & 14.351 \\
\textbf{LFW} & 0.3333 & 18.09 \\
\textbf{CIFAR-10} & 0.0913 & 4.258 \\
\bottomrule
\end{tabular}
\end{table}

\begin{figure}
    \centering
    \includegraphics[width=0.35\textwidth]{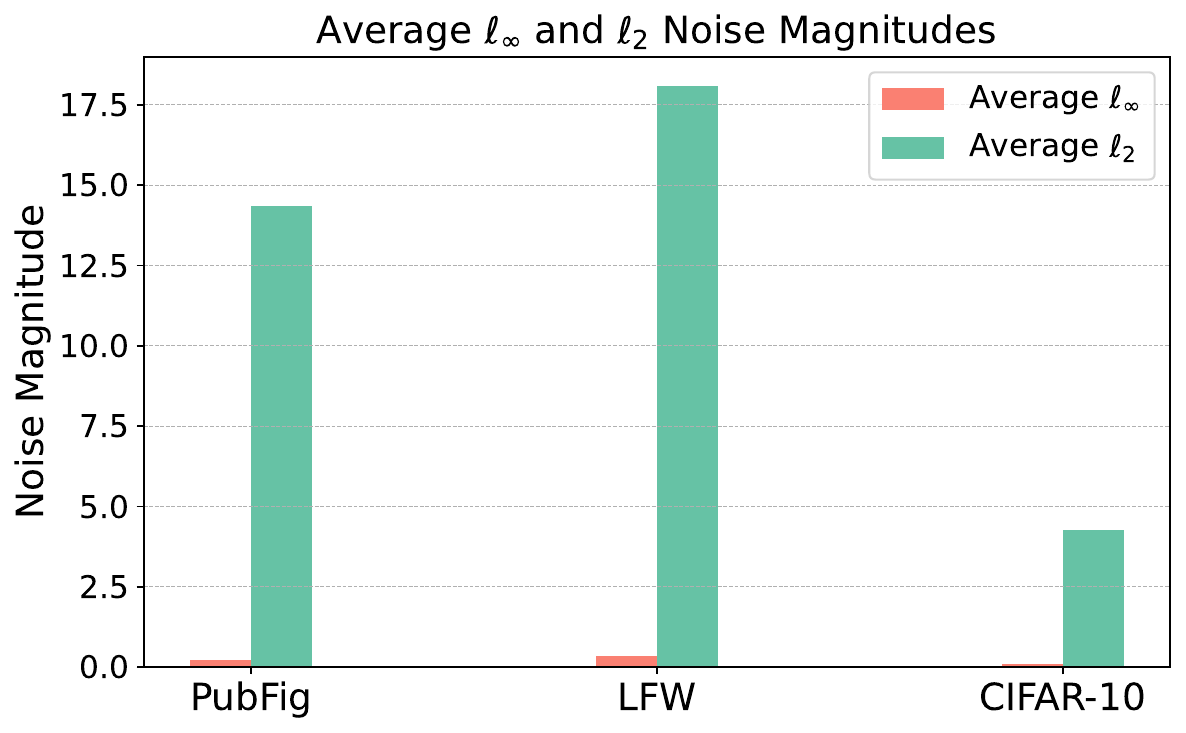}
    \caption{Average $\ell_\infty$ and $\ell_2$ corrective noise magnitudes for poisoned samples across PubFig, LFW, and CIFAR-10.}
    \label{fig11}
\end{figure}

To provide a more granular view of the distribution of perturbation magnitudes, Figure~\ref{fig12} shows histograms of noise norms $\ell_\infty$ and $\ell_2$ in all corrected poisoned samples in each dataset. The noise distributions are centered on small values, with low variance, suggesting that strong recovery is possible with consistently small perturbations.

\begin{figure}[ht]
    \centering
    \includegraphics[width=0.35\textwidth]{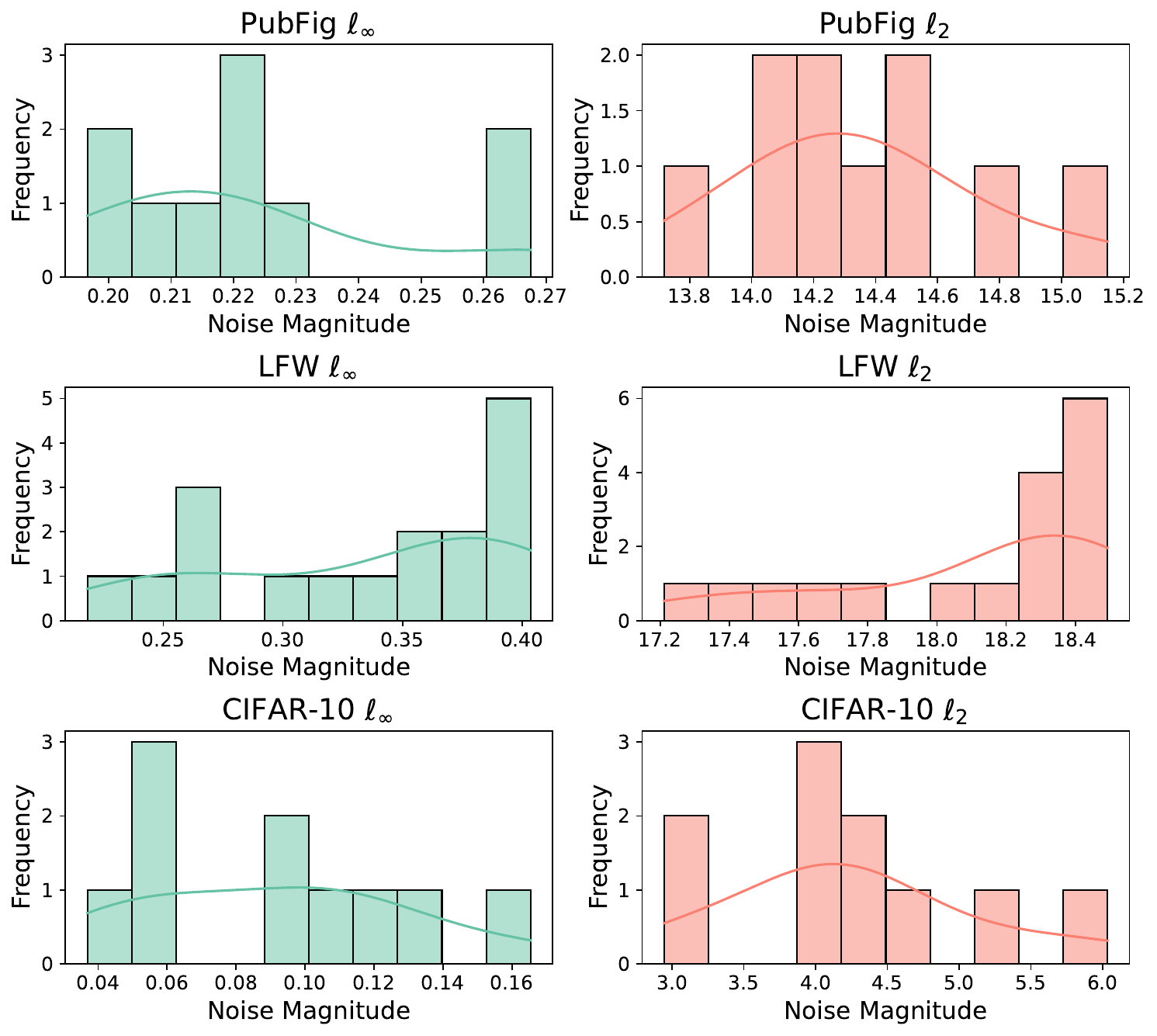}
    \caption{Histogram of $\ell_\infty$ and $\ell_2$ corrective noise magnitudes across PubFig, LFW, and CIFAR-10.}
    \label{fig12}
\end{figure}

To visualize the effect of recovery, Figures~\ref{fig13}, \ref{fig14}, and \ref{fig15} show poisoned images, their corrected versions, and the corresponding perturbation heatmaps for five samples from each dataset. The remaining samples are presented in Figures~\ref{fig19}--\ref{fig21}. These visualizations reveal that corrective PGD introduces sparse and localized modifications, focusing on regions likely associated with the backdoor trigger, while preserving the rest of the image content.

\begin{figure}
    \centering
    \includegraphics[width=0.35\textwidth]{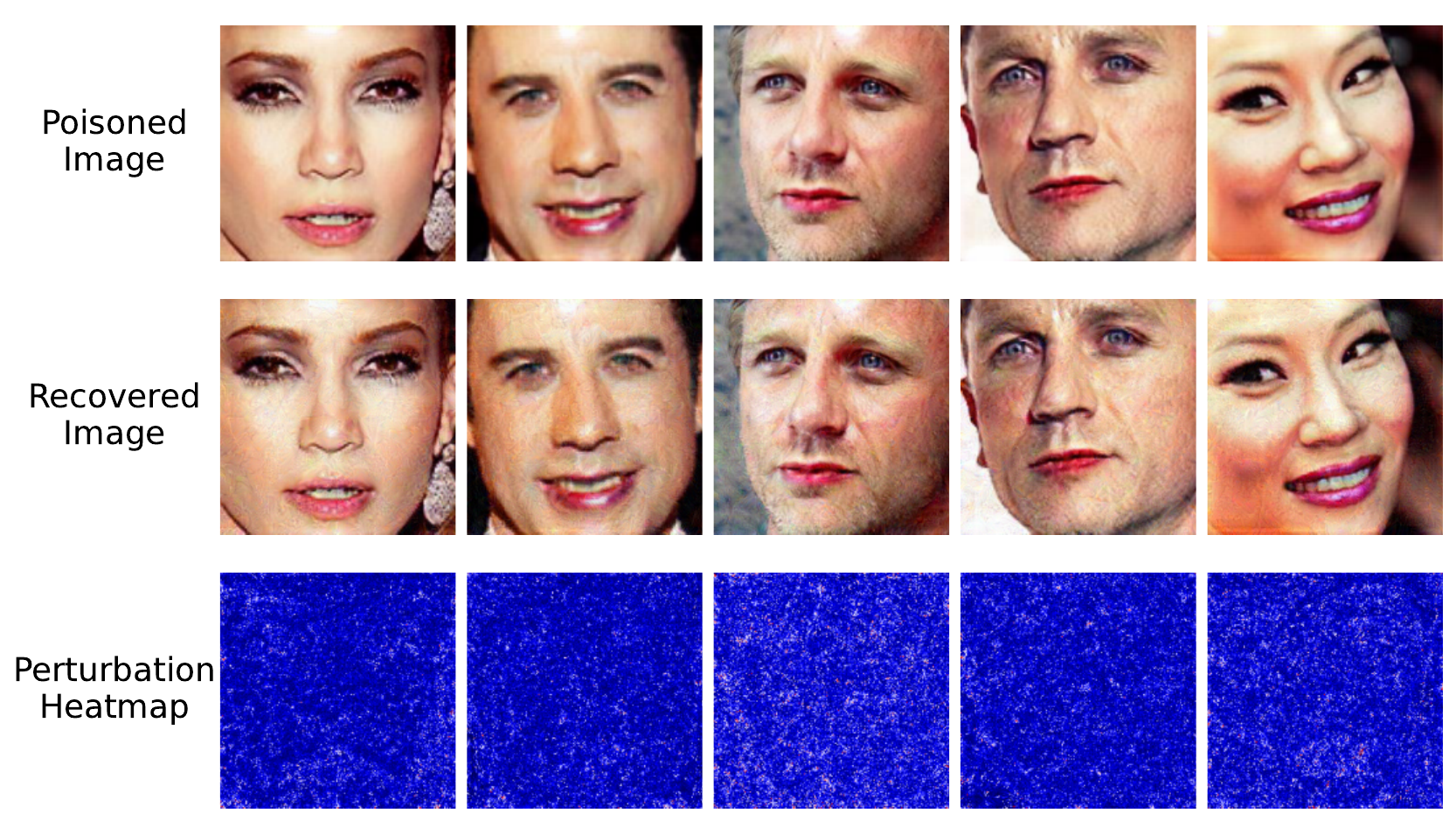}
    \caption{Sample poisoned and recovered images from the PubFig dataset, along with perturbation heatmaps.}
    \label{fig13}
\end{figure}

\begin{figure}
    \centering
    \includegraphics[width=0.35\textwidth]{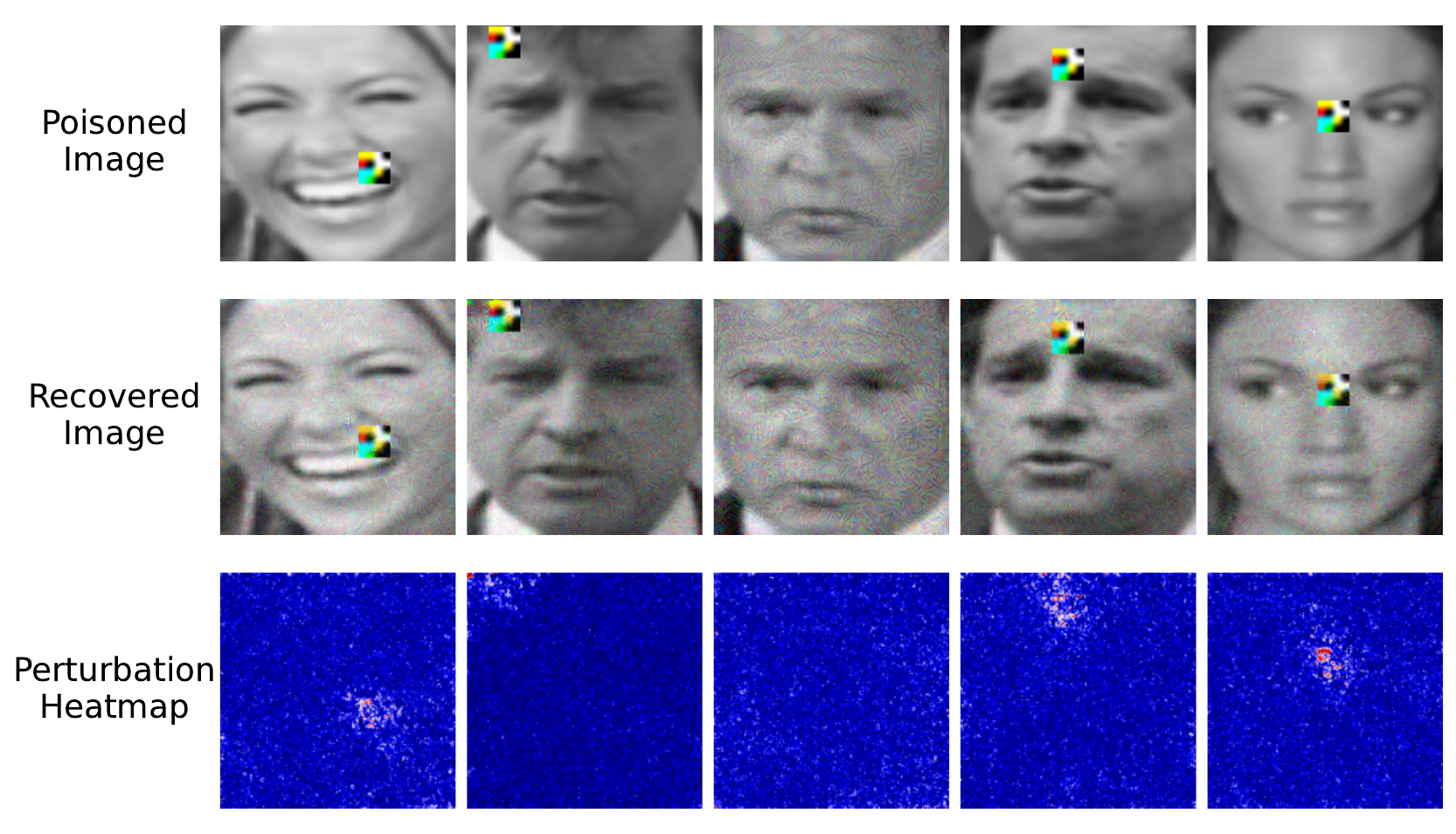}
    \caption{Sample poisoned and recovered images from the LFW dataset, along with perturbation heatmaps.}
    \label{fig14}
\end{figure}

\begin{figure}
    \centering
    \includegraphics[width=0.35\textwidth]{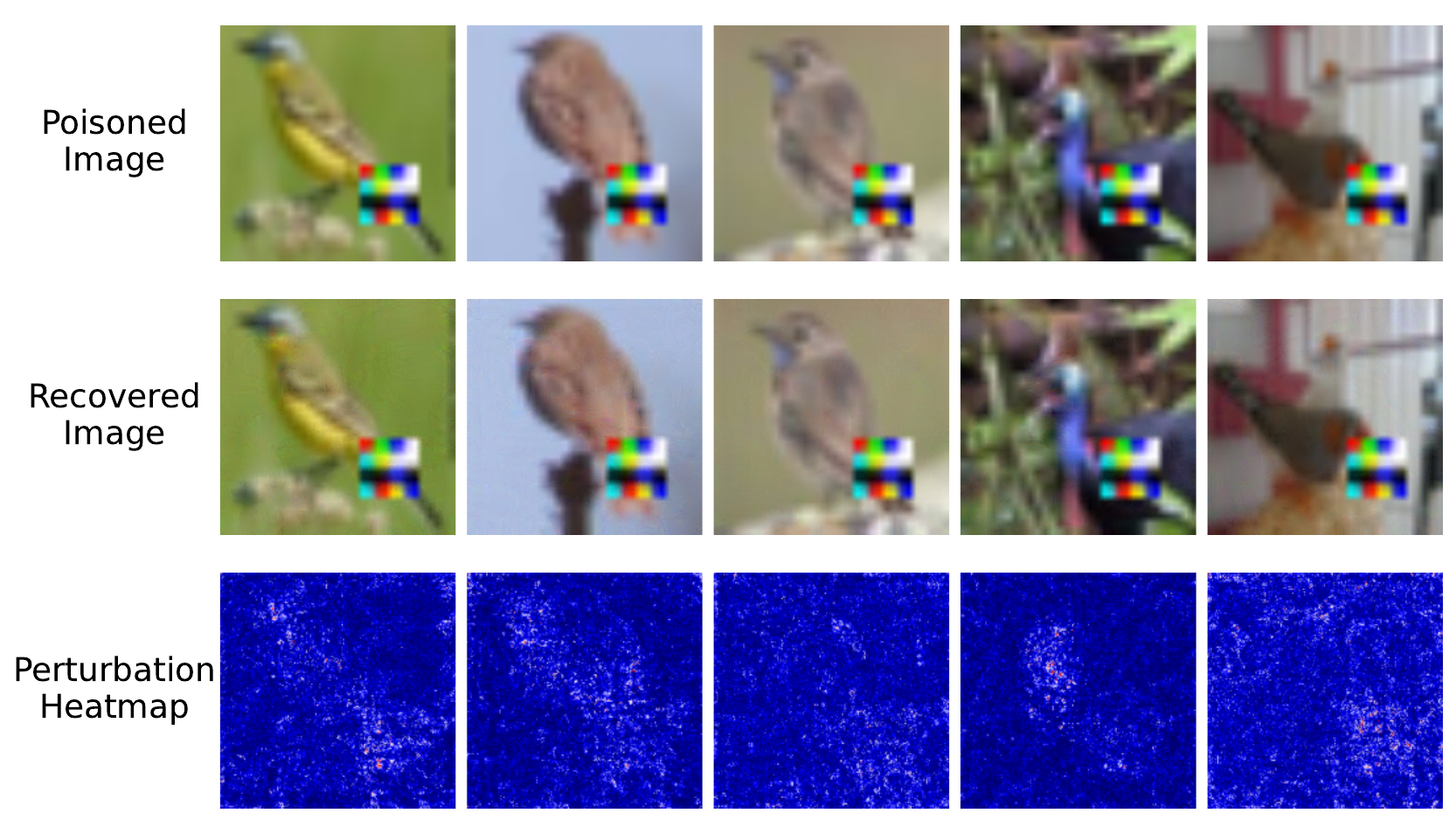}
    \caption{Sample poisoned and recovered images from the CIFAR-10 dataset, along with perturbation heatmaps.}
    \label{fig15}
\end{figure}

As discussed in Section~\ref{Methodology_CorrectivePGD}, the perturbation budget $\epsilon$ was dynamically selected for each dataset by estimating the maximum required correction $\Delta_{\max}$ using an unbounded PGD run and adding a small safety margin $\delta = 0.05$. This ensures that even strong triggers can be neutralized without excessive distortion. The specific values of $\Delta_{\max}$, $\epsilon$, step size $\alpha$, and iteration count $T$ used for each dataset are summarized in Table~\ref{tab10}.

\begin{table}[ht]
\centering
\caption{PGD hyperparameters for corrective recovery in each dataset.}
\label{tab10}
\resizebox{\linewidth}{!}{%
\begin{tabular}{lccccc}
\toprule
\textbf{Dataset} & \textbf{$\Delta_{max}$} & \textbf{$\delta$} & \textbf{$\epsilon = \Delta_{max} + \delta$} & \textbf{Iteration Steps ($T$)} & $\alpha = 1/T$\\
\midrule
\textbf{PubFig} & 0.5200 & & 0.5460 & & 0.002730 \\
\textbf{LFW} & 0.6300 & 0.05 & 0.6615 & 200 & 0.003307 \\
\textbf{CIFAR-10} & 0.1800 & & 0.1890 & & 0.000945 \\
\bottomrule
\end{tabular}
}
\end{table}

These results collectively demonstrate that our TrueBiometric can reliably and precisely reverse the effects of backdoor attacks with minimal, visually imperceptible perturbations. It maintains high recovery success on poisoned samples while preserving classification integrity for clean inputs, making it robust and safe.

\subsection{Computational Complexity}
\label{compute_complx}

The TrueBiometric framework consists of two sequential components that primarily govern its computational complexity: (1) the multimodal voting ensemble, and (2)  Corrective PGD on flagged images. For the first stage, each image is passed through a set of $M$ pre-trained VLMs. Each VLM performs a forward inference to classify whether the image is poisoned. Since VLMs are frozen and inference is constant in time, this step has a complexity of $O(M)$ per image, which is $O(1)$ in practice due to the small fixed value of $M$ (e.g., $M=5$). 
However, in practice, the current implementation relies on commercially available VLMs, accessed via external APIs. 

These commercial models introduce significant practical overheads, including inference latency, API response variability, privacy constraints, and rate limits, making the real-world computational cost non-trivial and potentially unpredictable.
For the second stage, the Corrective PGD routine operates only on the poisoned images flagged by the ensemble. The procedure first estimates the perturbation budget $\epsilon$ by applying an unbounded PGD to each flagged image, which takes $O(T)$ time per image, where $T$ is the number of PGD steps (typically fixed, e.g., $T=200$).

\begin{table*}[!ht]
  \centering
  \caption{Comparison of TrueBiometric with other detection/mitigation approaches.}
  \label{tab11}
  \resizebox{\linewidth}{!}{%
    \begin{tabular}{lcccccc}
      \toprule
      \textbf{Method} 
        & \textbf{Face‑} 
        & \textbf{Training‑} 
        & \textbf{Detection} 
        & \textbf{Mitigation} 
        & \textbf{Computational} 
        & \textbf{Accuracy} \\
      \textbf{} 
        & \textbf{specific} 
        & \textbf{free} 
        & \textbf{} 
        & \textbf{} 
        & \textbf{Complexity} 
        & \textbf{Drop} \\
      \midrule
      RAB~\cite{DBLP:journals/corr/abs-2003-08904}
        & \textcolor{red}{\ding{55}}
        & \textcolor{red}{\ding{55}}
        & \textcolor{green}{\ding{51}}
        & \textcolor{green}{\ding{51}}
        & High ($O(nK^2)$)
        & 6–16\% \\
      STRIP~\cite{DBLP:journals/corr/abs-1902-06531}
        & \textcolor{green}{\ding{51}}
        & \textcolor{green}{\ding{51}}
        & \textcolor{green}{\ding{51}}
        & \textcolor{red}{\ding{55}}
        & Moderate
        & $\approx$2\% \\
      TTBD~\cite{guan2024backdoor}
        & \textcolor{green}{\ding{51}}
        & \textcolor{red}{\ding{55}}
        & \textcolor{green}{\ding{51}}
        & \textcolor{green}{\ding{51}}
        & Moderate (inference‑
        & Minimal \\
        & & & & & time pruning) & \\
      BackdoorBench~\cite{wu2022backdoorbench}
        & Benchmark only
        & – 
        & – 
        & – 
        & – 
        & – \\
      MakeupAttack~\cite{sun2024makeupattack}
        & \textcolor{green}{\ding{51}}
        & –
        & \textcolor{red}{\ding{55}}
        & \textcolor{red}{\ding{55}}
        & –
        & – \\
      Face Forgery~\cite{han2023possible}
        & Forgery‑specific
        & –
        & \textcolor{red}{\ding{55}}
        & \textcolor{red}{\ding{55}}
        & –
        & – \\
      \textbf{TrueBiometric (OurApproach)}
        & \textcolor{green}{\ding{51}}
        & \textcolor{green}{\ding{51}}
        & \textcolor{green}{\ding{51}} (100\%)
        & \textcolor{green}{\ding{51}} (100\%)
        & Low ($O(N)$)
        & Negligible ($\approx$0\%) \\
      \bottomrule
    \end{tabular}
  }
\end{table*}

Once $\epsilon$ is computed, a second bounded PGD run is used to generate the corrected image, again taking $O(T)$ time. Hence, the total complexity for correcting each poisoned image is $O(T)$, which is also constant in practice. In the worst-case scenario, where all $N$ images are poisoned, the ensemble detection takes $O(N)$ total time, and the PGD-based correction adds another $O(NT)$, leading to a total worst-case complexity of $O(NT)$. However, since $T$ is fixed and independent of the dataset size, the overall complexity is simplified to linear time $O(N)$ with respect to the dataset size. This linear complexity makes TrueBiometric highly scalable and efficient, ensuring that it can be practically deployed for secure face recognition in large-scale biometric systems in the real world.

\subsection{Comparison with Existing Approaches}

We conduct a comparative analysis against several state-of-the-art backdoor defense methods. Table~\ref{tab11} presents a summary of this evaluation, emphasizing key distinctions in detection performance, mitigation success, computational overhead, and suitability for biometric applications. 

The Randomized Aggregated Backdoor (RAB) defense proposed by Zhang et al.~\cite{DBLP:journals/corr/abs-2003-08904} achieves certified robustness against backdoor attacks using randomized smoothing. However, RAB requires retraining multiple models and suffers from high computational complexity ($O(n \cdot K^2)$), resulting in substantial accuracy drops (6--16\%) even on simpler datasets, which limits its practical deployment in complex real-world scenarios. 

STRIP (STRong Intentional Perturbation)~\cite{DBLP:journals/corr/abs-1902-06531} detects poisoned inputs by analyzing the randomness of the model outputs. It operates in a training-free manner, providing efficient detection that is specifically suitable for face recognition. However, STRIP does not mitigate detected attacks and introduces a loss of accuracy of approximately 2\% in benign images. 

Test-Time Backdoor Defense (TTBD)~\cite{guan2024backdoor}, uses Shapley-based neuron pruning at inference time to detect and mitigate backdoor threats. Although effective, TTBD requires model adjustments during inference, introducing moderate computational overhead. BackdoorBench~\cite{wu2022backdoorbench}, a comprehensive benchmark evaluating more than 30 backdoor defenses. Although valuable for performance analysis, it does not propose a practical face-specific defense mechanism. 

The MakeupAttack~\cite{sun2024makeupattack} method highlights realistic threats to face recognition through natural facial modifications such as makeup. Similarly, Han et al.~\cite{han2023possible} investigated vulnerabilities of face forgery detection systems to backdoor attacks. Although these works emphasize realistic threats, they lack the corresponding defensive strategies. Compared to existing methods, TrueBiometric provides clear advantages. It was specifically tailored and evaluated for face datasets and effectively handles realistic backdoor triggers. It employs a set of multimodal VLMs for instant detection, eliminating the need for retraining and enabling seamless integration. It achieves a complete detection and mitigation rate (100\%), significantly outperforming existing approaches. It operates in linear complexity ($O(N)$), facilitating efficient deployment in real-world scenarios. Negligible reduction in accuracy on benign images, ensuring practical applicability without degrading the utility of the model.

\section{Related Works}
\label{related_works}

The existing literature on backdoor mitigation encompasses both data-level and model-level defenses; however, most solutions are reactive, computationally expensive, or fail to handle natural triggers. Early defenses such as STRIP~\cite{DBLP:journals/corr/abs-1902-06531} detect poisoned samples by evaluating output entropy under random perturbations. While effective against synthetic triggers, STRIP is limited to detection and struggles with semantic or natural-looking backdoors. Neural Cleanse~\cite{wang2019neural} reverse-engineers class-wise perturbations to identify potential triggers, but it assumes that the backdoor occupies a small, consistent region, which is often not the case for distributed or natural triggers. 

Neural Attention Distillation (NAD)~\cite{li2021neural} and Adversarial Neuron Pruning (ANP)~\cite{wu2021adversarial} go beyond detection by modifying model internals to remove backdoor-specific neurons. However, these model sanitization methods require access to the model's architecture and retraining on clean data, which is often impractical in real-world scenarios, especially for third-party models. Certified defenses such as Randomized Aggregation for Backdoor (RAB)~\cite{DBLP:journals/corr/abs-2003-08904} provide theoretical guarantees but are prohibitively expensive, making them unsuitable for large-scale face recognition systems or time-critical applications.

Recent work has shifted toward more realistic attack settings. MakeupAttack~\cite{sun2024makeupattack} introduced a family of attacks using natural cosmetic overlays such as lipstick or blush to inject imperceptible triggers into facial images. These attacks are harder to detect due to their semantic alignment with the image content and pose significant challenges to existing defenses. Similarly, WaNet~\cite{nguyen2021wanet} proposes a wrapping-based trigger mechanism that gently distorts the image geometry rather than overlaying a patch, making it imperceptible to both human observers and conventional defenses. Another example is Reflection Backdoor (Refool)~\cite{liu2020reflection}, which embeds subtle light-reflection artifacts in face images. These reflections simulate natural environmental effects and can consistently activate backdoors without altering primary facial features. These attacks are harder to detect due to their semantic alignment with the image content and pose significant challenges to existing defenses. In response, multimodal approaches have emerged. 

BDetCLIP~\cite{niu2024bdetclip} uses contrastive prompting in VLMs such as CLIP to detect poisoned samples at test time by measuring inconsistencies between textual descriptions and visual inputs. Similarly, recent work by Huang et al.~\cite{huang2025detecting} applies density-based estimation to detect poisoned samples during CLIP pretraining. Although these methods demonstrate strong detection capability, they are detection-only and do not attempt to correct or sanitize the identified poisoned inputs. Thus, the challenge of developing an accurate and efficient method that can both detect and neutralize more recent, sophisticated backdoor attacks remains largely unaddressed in the literature.

\section{Conclusion}
\label{conclusion}

We propose TrueBiometric, a novel framework for mitigating backdoor attacks in face recognition systems, utilizing a multimodal ensemble of VLMs for robust detection of poisoned images, coupled with a corrective PGD-based recovery mechanism. TrueBiometric operates without needing access to trusted training data, prior knowledge of the trigger structure, or model retraining, making it applicable in realistic deployment settings. Our extensive evaluation across multiple datasets (PubFig, LFW, CIFAR-10) demonstrates that TrueBiometric achieves 100\% recovery success on poisoned images while preserving clean image predictions with negligible distortion. Compared to existing defenses, TrueBiometric offers superior detection accuracy, minimal accuracy drop, and low computational complexity, making it well-suited for large-scale, privacy-sensitive biometric authentication systems.

\textbf{Future Work.}
Future directions include extending TrueBiometric to other biometric modalities such as iris, fingerprint, and palm-print recognition, and adapting it to handle emerging attack variants, including clean-label and semantic backdoors, in unconstrained environments.


\appendix
\section{Appendix}
\label{appendix}

In this appendix, we provide the complete per-image and aggregate statistics supporting the main results. Tables~\ref{tab:pubfig_decisions}–\ref{tab:confusion_percentages_cifar10} present, for each dataset (PubFig, LFW, CIFAR-10), the per-image decisions of each VLM and the majority vote, the corresponding confusion matrix counts, and their normalized percentages. Table~\ref{tab:ident_counts} summarizes the identification counts per-dataset (poisoned vs. clean) for each VLM and the ensemble. These detailed results enable full reproducibility and provide additional information on the behavior of individual models beyond the summary presented in the main text. We have also provided additional image examples to illustrate how individual VLMs can differ from majority vote decisions, as well as further demonstrations of our recovery method. Figure~\ref{fig16} presents clean images correctly identified by the majority vote but misclassified as poisoned by some individual VLMs. Figure~\ref{fig17} shows all the poisoned images correctly identified by the majority vote but misclassified as clean by certain VLMs, specifically four images each from the PubFig and Cifar-10 datasets, but no examples occurred in the LFW dataset. Figure~\ref{fig18} includes all false positive cases where images are genuinely clean but incorrectly marked as poisoned by the majority vote (five from PubFig, one from LFW, and two from the Cifar-10 dataset). Figures~\ref{fig19}--\ref{fig21} provide additional visual examples for the PubFig, LFW, and Cifar-10 datasets, respectively, illustrating the original images, corresponding poisoned images, their recovered image versions after applying our corrective method, and perturbation heatmaps. These additional examples complement the main results and help validate the robustness and effectiveness of our methodology. To aid in clarity and ensure consistent interpretation, Table~\ref{tab:notation} summarizes all the mathematical notation used throughout the paper along with their definitions. To assess whether VLMs can directly regenerate original (clean) images after identifying poisoned samples, we performed a set of feasibility experiments across the datasets using the VLMs. Table~\ref{tab:llm_regeneration_summary} summarizes the overall performance and responses of each VLM. The results indicate that while some models attempted regeneration, none successfully restored the original images accurately, highlighting practical limitations and supporting the adoption of the corrective noise approach presented in our main methodology. Some examples of original (clean) images regenerated by VLMs are illustrated in Figure~\ref{fig24}--\ref{fig26}.

\begin{figure}[H]
    \centering
    \includegraphics[width=0.5\textwidth]{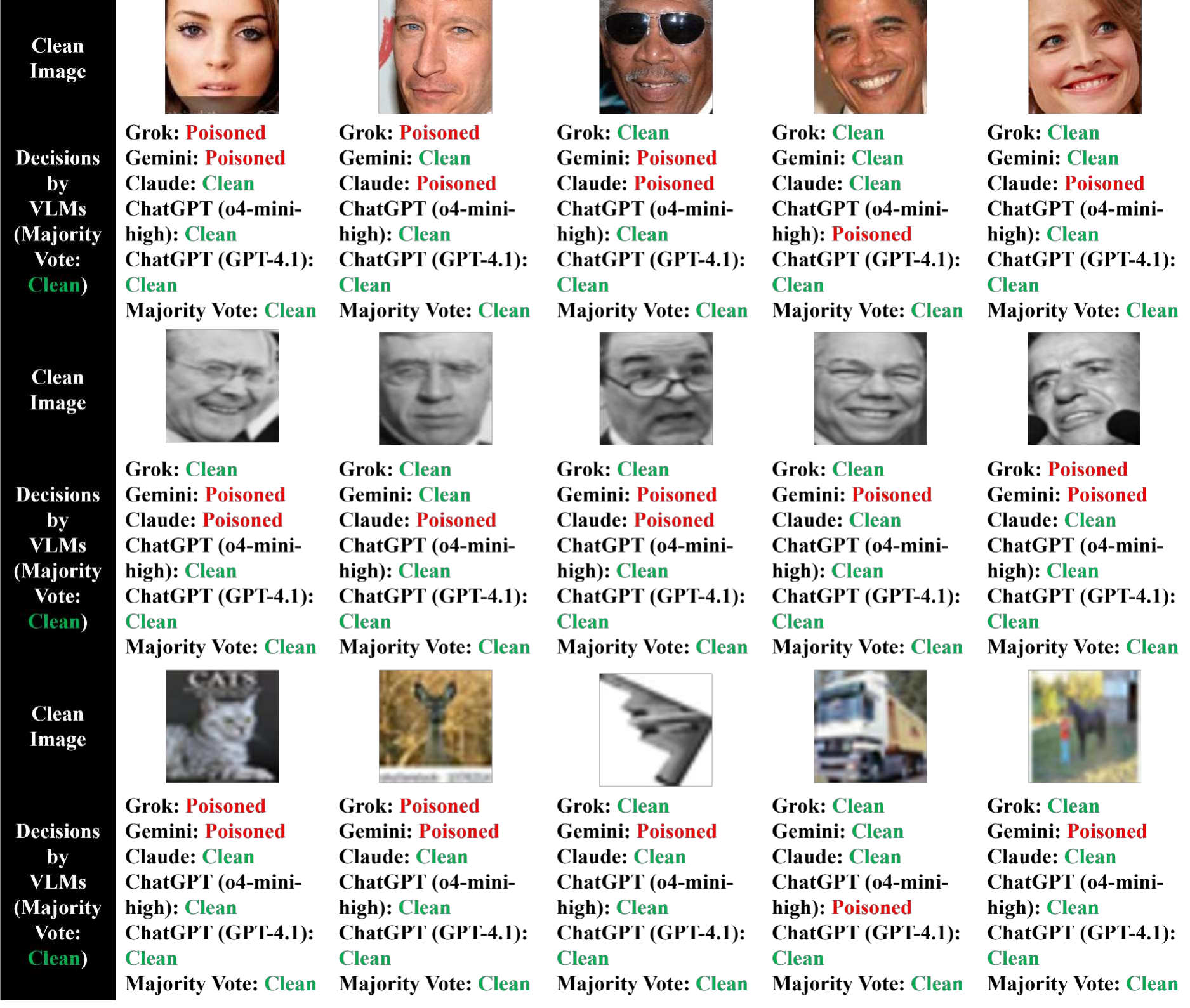}
    \caption{Clean images correctly identified by majority vote but flagged as poisoned by some VLMs.}
    \label{fig16}
\end{figure}

\begin{figure}[H]
    \centering
    \includegraphics[width=0.5\textwidth]{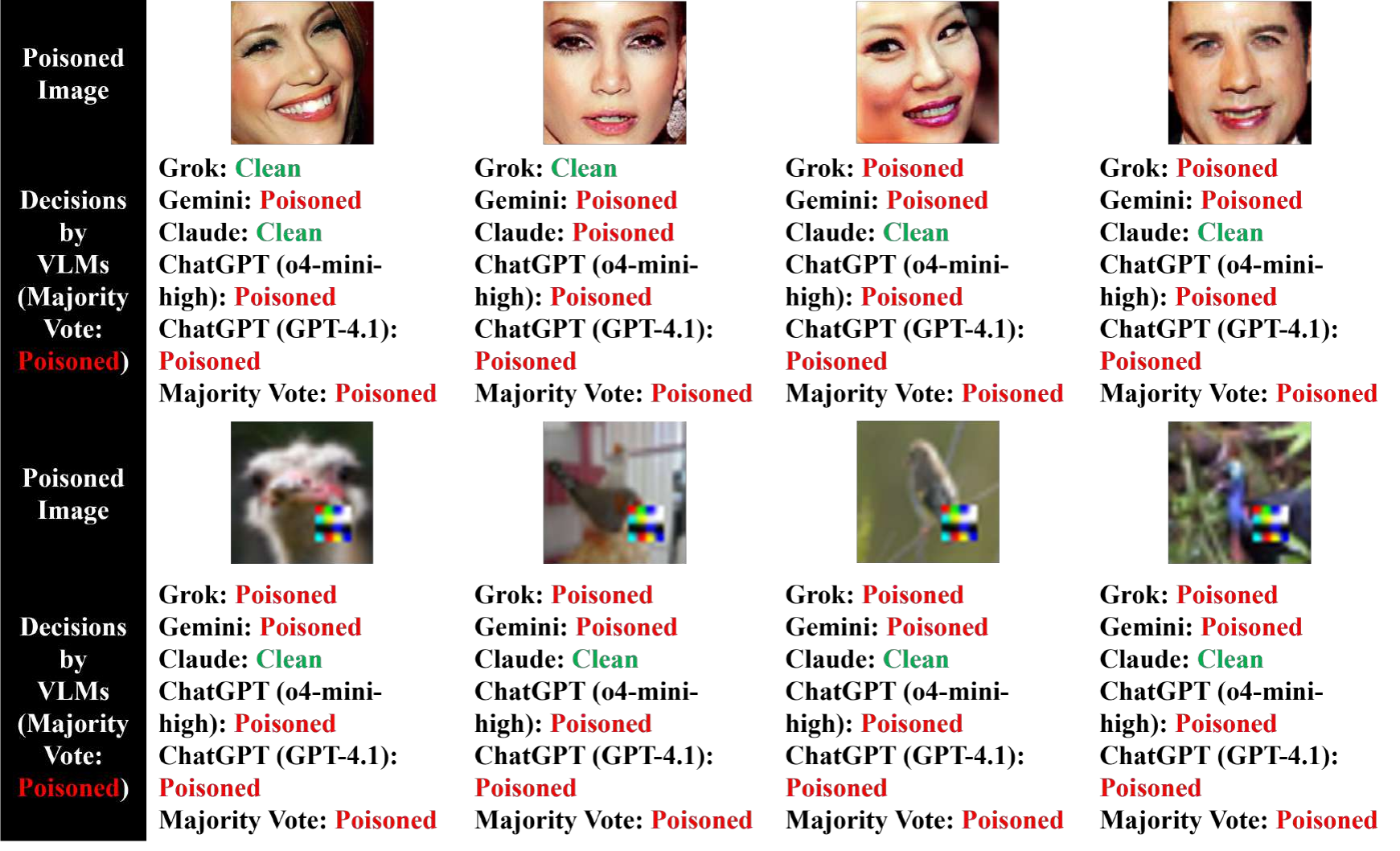}
    \caption{Poisoned images correctly identified by majority vote but flagged as clean by some VLMs.}
    \label{fig17}
\end{figure}

\begin{figure}[H]
    \centering
    \includegraphics[width=0.5\textwidth]{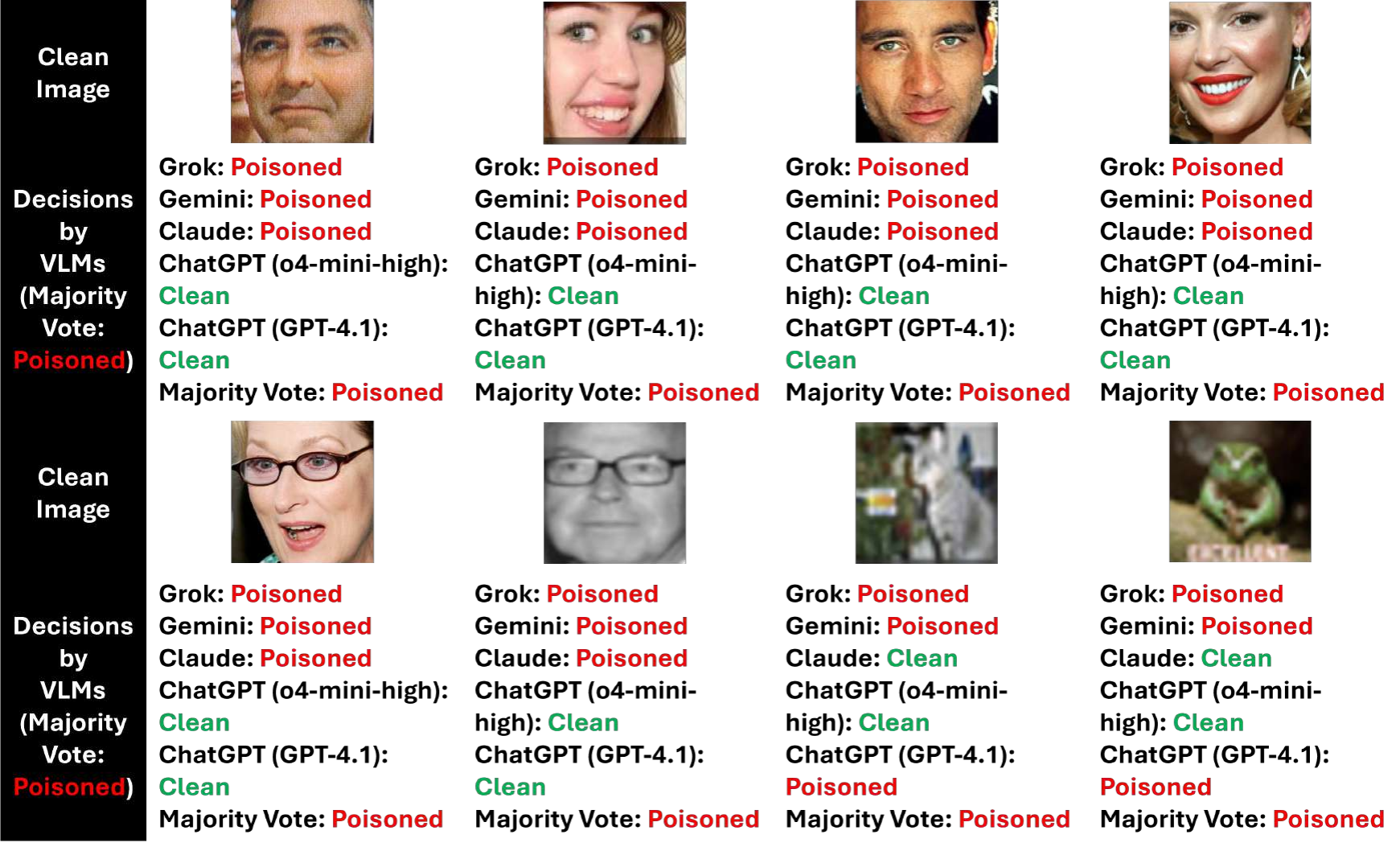}
    \caption{False positive--clean images incorrectly identified as poisoned by the majority vote.}
    \label{fig18}
\end{figure}

\begin{table}[H]
  \centering
  \caption{Decisions on the PubFig dataset by each VLM and by majority vote.}
  \label{tab:pubfig_decisions}
  \resizebox{\columnwidth}{!}{%
  \begin{tabular}{@{}l l *{6}{c}@{}}
    \toprule
    \textbf{Image} 
    & \textbf{Type} 
    & \textbf{Grok} 
    & \textbf{Gemini} 
    & \textbf{Claude} 
    & \textbf{ChatGPT} 
    & \textbf{ChatGPT} 
    & \textbf{Majority} \\

    \textbf{Number}
    & \textbf{} 
    & \textbf{(Grok 3)} 
    & \textbf{(2.5 Pro)} 
    & \textbf{(Sonnet~4)} 
    & \textbf{(o4-mini-high)} 
    & \textbf{(GPT\,4.1)} 
    & \textbf{Vote} \\
    \midrule
  
    Original 1 & Clean & Clean & Clean & Poisoned & Clean & Clean & Clean \\
    Original 2 & Clean & Clean & Poisoned & Poisoned & Clean & Clean & Clean \\
    Original 3 & Clean & Clean & Clean & Clean & Clean & Clean & Clean \\
    Original 4 & Clean & Clean & Clean & Poisoned & Clean & Clean & Clean \\
    Original 5 & Clean & Clean & Poisoned & Poisoned & Clean & Clean & Clean \\
    Original 6 & Clean & Poisoned & Poisoned & Clean & Clean & Clean & Clean \\
    Original 7 & Clean & Clean & Clean & Poisoned & Clean & Clean & Clean \\
    Original 8 & Clean & Clean & Poisoned & Poisoned & Clean & Clean & Clean \\
    Original 9 & Clean & Clean & Clean & Clean & Clean & Clean & Clean \\
    Original 10 & Clean & Clean & Poisoned & Poisoned & Clean & Clean & Clean \\
    Original 11 & Clean & Clean & Clean & Clean & Clean & Clean & Clean \\
    Original 12 & Clean & Clean & Poisoned & Clean & Clean & Clean & Clean \\
    Original 13 & Clean & Clean & Poisoned & Poisoned & Clean & Clean & Clean \\
    Original 14 & Clean & Clean & Poisoned & Poisoned & Clean & Clean & Clean \\
    Original 15 & Clean & Clean & Clean & Poisoned & Clean & Clean & Clean \\
    Original 16 & Clean & Clean & Clean & Clean & Clean & Clean & Clean \\
    Original 17 & Clean & Clean & Poisoned & Poisoned & Clean & Clean & Clean \\
    Original 18 & Clean & Clean & Clean & Poisoned & Clean & Clean & Clean \\
    Original 19 & Clean & Clean & Poisoned & Clean & Clean & Clean & Clean \\
    Original 20 & Clean & Clean & Clean & Poisoned & Clean & Clean & Clean \\
    Original 21 & Clean & Clean & Clean & Poisoned & Clean & Clean & Clean \\
    Original 22 & Clean & Clean & Clean & Clean & Clean & Clean & Clean \\
    Original 23 & Clean & Clean & Clean & Clean & Clean & Clean & Clean \\
    Original 24 & Clean & Clean & Poisoned & Poisoned & Clean & Clean & Clean \\
    Original 25 & Clean & Clean & Clean & Clean & Clean & Clean & Clean \\
    Original 26 & Clean & Clean & Clean & Clean & Clean & Clean & Clean \\
    Original 27 & Clean & Clean & Clean & Poisoned & Clean & Clean & Clean \\
    Original 28 & Clean & Clean & Clean & Poisoned & Clean & Clean & Clean \\
    Original 29 & Clean & Poisoned & Poisoned & Poisoned & Clean & Clean & Poisoned \\
    Original 30 & Clean & Clean & Clean & Clean & Clean & Clean & Clean \\
    Original 31 & Clean & Clean & Poisoned & Poisoned & Clean & Clean & Clean \\
    Original 32 & Clean & Clean & Poisoned & Poisoned & Clean & Clean & Clean \\
    Original 33 & Clean & Clean & Clean & Clean & Clean & Clean & Clean \\
    Original 34 & Clean & Poisoned & Clean & Poisoned & Clean & Clean & Clean \\
    Original 35 & Clean & Clean & Clean & Clean & Clean & Clean & Clean \\
    Original 36 & Clean & Clean & Poisoned & Poisoned & Clean & Clean & Clean \\
    Original 37 & Clean & Clean & Clean & Clean & Clean & Clean & Clean \\
    Original 38 & Clean & Clean & Poisoned & Clean & Clean & Clean & Clean \\
    Original 39 & Clean & Clean & Poisoned & Clean & Clean & Clean & Clean \\
    Original 40 & Clean & Clean & Clean & Clean & Clean & Clean & Clean \\
    Original 41 & Clean & Clean & Clean & Clean & Clean & Clean & Clean \\
    Original 42 & Clean & Clean & Clean & Clean & Clean & Clean & Clean \\
    Original 43 & Clean & Clean & Clean & Clean & Clean & Clean & Clean \\
    Original 44 & Clean & Poisoned & Poisoned & Poisoned & Clean & Clean & Poisoned \\
    Original 45 & Clean & Poisoned & Poisoned & Poisoned & Clean & Clean & Poisoned \\
    Original 46 & Clean & Clean & Clean & Clean & Clean & Clean & Clean \\
    Original 47 & Clean & Clean & Clean & Clean & Clean & Clean & Clean \\
    Original 48 & Clean & Poisoned & Poisoned & Clean & Clean & Clean & Clean \\
    Original 49 & Clean & Clean & Poisoned & Poisoned & Clean & Clean & Clean \\
    Original 50 & Clean & Clean & Poisoned & Poisoned & Clean & Clean & Clean \\
    Original 51 & Clean & Clean & Clean & Clean & Clean & Clean & Clean \\
    Original 52 & Clean & Clean & Poisoned & Poisoned & Clean & Clean & Clean \\
    Original 53 & Clean & Clean & Clean & Clean & Clean & Clean & Clean \\
    Original 54 & Clean & Clean & Poisoned & Clean & Clean & Clean & Clean \\
    Original 55 & Clean & Clean & Clean & Poisoned & Clean & Clean & Clean \\
    Original 56 & Clean & Clean & Clean & Clean & Clean & Clean & Clean \\
    Original 57 & Clean & Poisoned & Clean & Clean & Clean & Clean & Clean \\
    Original 58 & Clean & Clean & Poisoned & Poisoned & Clean & Clean & Clean \\
    Original 59 & Clean & Poisoned & Poisoned & Clean & Clean & Clean & Clean \\
    Original 60 & Clean & Poisoned & Poisoned & Poisoned & Clean & Clean & Poisoned \\
    Original 61 & Clean & Clean & Clean & Poisoned & Clean & Clean & Clean \\
    Original 62 & Clean & Poisoned & Poisoned & Poisoned & Clean & Clean & Poisoned \\
    Original 63 & Clean & Clean & Poisoned & Poisoned & Clean & Clean & Clean \\
    Original 64 & Clean & Clean & Clean & Clean & Clean & Clean & Clean \\
    Original 65 & Clean & Clean & Clean & Poisoned & Clean & Clean & Clean \\
    Original 66 & Clean & Clean & Poisoned & Poisoned & Clean & Clean & Clean \\
    Original 67 & Clean & Clean & Clean & Clean & Clean & Clean & Clean \\
    Original 68 & Clean & Clean & Clean & Clean & Clean & Clean & Clean \\
    Original 69 & Clean & Clean & Poisoned & Poisoned & Clean & Clean & Clean \\
    Original 70 & Clean & Clean & Clean & Poisoned & Clean & Clean & Clean \\
    Original 71 & Clean & Clean & Poisoned & Poisoned & Clean & Clean & Clean \\
    Original 72 & Clean & Clean & Clean & Clean & Clean & Clean & Clean \\
    Original 73 & Clean & Clean & Poisoned & Poisoned & Clean & Clean & Clean \\
    Original 74 & Clean & Clean & Clean & Clean & Clean & Clean & Clean \\
    Original 75 & Clean & Clean & Clean & Clean & Clean & Clean & Clean \\
    Original 76 & Clean & Poisoned & Clean & Poisoned & Clean & Clean & Clean \\
    Original 77 & Clean & Clean & Clean & Clean & Clean & Clean & Clean \\
    Original 78 & Clean & Clean & Poison & Poison & Clean & Clean & Clean \\
    Original 79 & Clean & Clean & Poisoned & Clean & Clean & Clean & Clean \\
    Original 80 & Clean & Clean & Clean & Clean & Clean & Clean & Clean \\
    Original 81 & Clean & Clean & Clean & Clean & Clean & Clean & Clean \\
    Original 82 & Clean & Clean & Clean & Clean & Clean & Clean & Clean \\
    Original 83 & Clean & Clean & Clean & Clean & Clean & Clean & Clean \\
    Original 84 & Clean & Clean & Clean & Poisoned & Clean & Clean & Clean \\
    Original 85 & Clean & Clean & Clean & Clean & Clean & Clean & Clean \\
    Original 86 & Clean & Clean & Clean & Poisoned & Clean & Clean & Clean \\
    Original 87 & Clean & Clean & Clean & Clean & Clean & Clean & Clean \\
    Original 88 & Clean & Clean & Clean & Clean & Clean & Clean & Clean \\
    Original 89 & Clean & Poisoned & Clean & Clean & Clean & Clean & Clean \\
    Original 90 & Clean & Clean & Clean & Clean & Poisoned & Clean & Clean \\
    Poison 1 & Poisoned & Poisoned & Poisoned & Poisoned & Poisoned & Poisoned & Poisoned \\
    Poison 2 & Poisoned & Poisoned & Poisoned & Poisoned & Poisoned & Poisoned & Poisoned \\
    Poison 3 & Poisoned & Poisoned & Poisoned & Poisoned & Poisoned & Poisoned & Poisoned \\
    Poison 4 & Poisoned & Poisoned & Poisoned & Clean & Poisoned & Poisoned & Poisoned \\
    Poison 5 & Poisoned & Poisoned & Poisoned & Clean & Poisoned & Poisoned & Poisoned \\
    Poison 6 & Poisoned & Poisoned & Poisoned & Poisoned & Poisoned & Poisoned & Poisoned \\
    Poison 7 & Poisoned & Clean & Poisoned & Clean & Poisoned & Poisoned & Poisoned \\
    Poison 8 & Poisoned & Clean & Poisoned & Poisoned & Poisoned & Poisoned & Poisoned \\
    Poison 9 & Poisoned & Poisoned & Poisoned & Poisoned & Poisoned & Poisoned & Poisoned \\
    Poison 10 & Poisoned & Poisoned & Poisoned & Poisoned & Poisoned & Poisoned & Poisoned \\
    \bottomrule
  \end{tabular}
  }
\end{table}

\begin{table}[H]
    \centering
    \caption{Decisions on the LFW dataset by each VLM and by majority vote.}
    \label{tab:lfw_decisions}
    \resizebox{\columnwidth}{!}{%
    \begin{tabular}{@{} l l *{6}{c} @{}}
    \toprule
    \textbf{Image} 
    & \textbf{Type} 
    & \textbf{Grok} 
    & \textbf{Gemini} 
    & \textbf{Claude} 
    & \textbf{ChatGPT} 
    & \textbf{ChatGPT} 
    & \textbf{Majority} \\

    \textbf{Number}
    & \textbf{} 
    & \textbf{(Grok 3)} 
    & \textbf{(2.5 Pro)} 
    & \textbf{(Sonnet~4)} 
    & \textbf{(o4-mini-high)} 
    & \textbf{(GPT\,4.1)} 
    & \textbf{Vote} \\
    \midrule

    Original 1 & Clean & Poisoned & Poisoned & Clean & Clean & Clean & Clean \\
    Original 2 & Clean & Poisoned & Clean & Clean & Clean & Clean & Clean \\
    Original 3 & Clean & Poisoned & Clean & Clean & Clean & Clean & Clean \\
    Original 4 & Clean & Poisoned & Clean & Clean & Clean & Clean & Clean \\
    Original 5 & Clean & Poisoned & Poisoned & Poisoned & Clean & Clean & Poisoned \\
    Original 6 & Clean & Poisoned & Clean & Clean & Clean & Clean & Clean \\
    Original 7 & Clean & Poisoned & Poisoned & Clean & Clean & Clean & Clean \\
    Original 8 & Clean & Clean & Poisoned & Poisoned & Clean & Clean & Clean \\
    Original 9 & Clean & Clean & Clean & Clean & Clean & Clean & Clean \\
    Original 10 & Clean & Clean & Clean & Clean & Clean & Clean & Clean \\
    Original 11 & Clean & Clean & Clean & Clean & Clean & Clean & Clean \\
    Original 12 & Clean & Clean & Clean & Clean & Clean & Clean & Clean \\
    Original 13 & Clean & Clean & Poisoned & Clean & Clean & Clean & Clean \\
    Original 14 & Clean & Clean & Poisoned & Poisoned & Clean & Clean & Clean \\
    Original 15 & Clean & Clean & Poisoned & Poisoned & Clean & Clean & Clean \\
    Original 16 & Clean & Clean & Clean & Clean & Clean & Clean & Clean \\
    Original 17 & Clean & Clean & Clean & Clean & Clean & Clean & Clean \\
    Original 18 & Clean & Clean & Clean & Clean & Clean & Clean & Clean \\
    Original 19 & Clean & Clean & Clean & Clean & Clean & Clean & Clean \\
    Original 20 & Clean & Clean & Clean & Clean & Clean & Clean & Clean \\
    Original 21 & Clean & Clean & Clean & Clean & Clean & Clean & Clean \\
    Original 22 & Clean & Clean & Clean & Poisoned & Clean & Clean & Clean \\
    Original 23 & Clean & Clean & Clean & Clean & Clean & Clean & Clean \\
    Original 24 & Clean & Clean & Clean & Clean & Clean & Clean & Clean \\
    Original 25 & Clean & Clean & Poisoned & Poisoned & Clean & Clean & Clean \\
    Original 26 & Clean & Poisoned & Clean & Clean & Clean & Clean & Clean \\
    Original 27 & Clean & Clean & Poisoned & Clean & Clean & Clean & Clean \\
    Original 28 & Clean & Clean & Clean & Clean & Clean & Clean & Clean \\
    Original 29 & Clean & Clean & Clean & Clean & Clean & Clean & Clean \\
    Original 30 & Clean & Clean & Clean & Clean & Clean & Clean & Clean \\
    Original 31 & Clean & Clean & Clean & Clean & Clean & Clean & Clean \\
    Original 32 & Clean & Clean & Poisoned & Clean & Clean & Clean & Clean \\
    Original 33 & Clean & Clean & Clean & Clean & Clean & Clean & Clean \\
    Original 34 & Clean & Clean & Clean & Clean & Clean & Clean & Clean \\
    Original 35 & Clean & Clean & Clean & Clean & Clean & Clean & Clean \\
    Original 36 & Clean & Clean & Clean & Clean & Clean & Clean & Clean \\
    Original 37 & Clean & Clean & Clean & Clean & Clean & Clean & Clean \\
    Original 38 & Clean & Clean & Clean & Clean & Clean & Clean & Clean \\
    Original 39 & Clean & Clean & Clean & Clean & Clean & Clean & Clean \\
    Original 40 & Clean & Clean & Clean & Clean & Clean & Clean & Clean \\
    Original 41 & Clean & Clean & Clean & Clean & Clean & Clean & Clean \\
    Original 42 & Clean & Clean & Clean & Clean & Clean & Clean & Clean \\
    Original 43 & Clean & Clean & Poisoned & Clean & Clean & Clean & Clean \\
    Original 44 & Clean & Clean & Clean & Clean & Clean & Clean & Clean \\
    Original 45 & Clean & Clean & Clean & Clean & Clean & Clean & Clean \\
    Original 46 & Clean & Clean & Clean & Clean & Clean & Clean & Clean \\
    Original 47 & Clean & Clean & Clean & Clean & Clean & Clean & Clean \\
    Original 48 & Clean & Clean & Poisoned & Clean & Clean & Clean & Clean \\
    Original 49 & Clean & Clean & Clean & Clean & Clean & Clean & Clean \\
    Original 50 & Clean & Clean & Clean & Clean & Clean & Clean & Clean \\
    Original 51 & Clean & Clean & Clean & Clean & Clean & Clean & Clean \\
    Original 52 & Clean & Clean & Clean & Clean & Clean & Clean & Clean \\
    Original 53 & Clean & Clean & Clean & Clean & Clean & Clean & Clean \\
    Original 54 & Clean & Clean & Clean & Clean & Clean & Clean & Clean \\
    Original 55 & Clean & Clean & Poisoned & Clean & Clean & Clean & Clean \\
    Original 56 & Clean & Clean & Clean & Clean & Clean & Clean & Clean \\
    Original 57 & Clean & Clean & Clean & Clean & Clean & Clean & Clean \\
    Original 58 & Clean & Clean & Clean & Clean & Clean & Clean & Clean \\
    Original 59 & Clean & Clean & Clean & Clean & Clean & Clean & Clean \\
    Original 60 & Clean & Clean & Poisoned & Clean & Clean & Clean & Clean \\
    Original 61 & Clean & Clean & Clean & Clean & Clean & Clean & Clean \\
    Original 62 & Clean & Clean & Clean & Clean & Clean & Clean & Clean \\
    Original 63 & Clean & Clean & Clean & Clean & Clean & Clean & Clean \\
    Original 64 & Clean & Poisoned & Clean & Clean & Clean & Clean & Clean \\
    Original 65 & Clean & Clean & Poisoned & Clean & Clean & Clean & Clean \\
    Original 66 & Clean & Poisoned & Poisoned & Clean & Clean & Clean & Clean \\
    Original 67 & Clean & Clean & Clean & Clean & Clean & Clean & Clean \\
    Original 68 & Clean & Poisoned & Poisoned & Clean & Clean & Clean & Clean \\
    Original 69 & Clean & Poisoned & Poisoned & Clean & Clean & Clean & Clean \\
    Original 70 & Clean & Clean & Clean & Clean & Clean & Clean & Clean \\
    Original 71 & Clean & Clean & Poisoned & Clean & Clean & Clean & Clean \\
    Original 72 & Clean & Poisoned & Clean & Clean & Clean & Clean & Clean \\
    Original 73 & Clean & Clean & Clean & Clean & Clean & Clean & Clean \\
    Original 74 & Clean & Poisoned & Poisoned & Clean & Clean & Clean & Clean \\
    Original 75 & Clean & Clean & Clean & Clean & Clean & Clean & Clean \\
    Original 76 & Clean & Clean & Clean & Clean & Clean & Clean & Clean \\
    Original 77 & Clean & Clean & Clean & Clean & Clean & Clean & Clean \\
    Original 78 & Clean & Clean & Clean & Clean & Clean & Clean & Clean \\
    Original 79 & Clean & Clean & Clean & Clean & Clean & Clean & Clean \\
    Original 80 & Clean & Clean & Clean & Clean & Clean & Clean & Clean \\
    Original 81 & Clean & Poisoned & Clean & Clean & Clean & Clean & Clean \\
    Original 82 & Clean & Poisoned & Poisoned & Clean & Clean & Clean & Clean \\
    Original 83 & Clean & Poisoned & Poisoned & Clean & Clean & Clean & Clean \\
    Patch 1 & Poisoned & Poisoned & Poisoned & Poisoned & Poisoned & Poisoned & Poisoned \\
    Patch 2 & Poisoned & Poisoned & Poisoned & Poisoned & Poisoned & Poisoned & Poisoned \\
    Patch 3 & Poisoned & Poisoned & Poisoned & Poisoned & Poisoned & Poisoned & Poisoned \\
    Patch 4 & Poisoned & Poisoned & Poisoned & Poisoned & Poisoned & Poisoned & Poisoned \\
    Patch 5 & Poisoned & Poisoned & Poisoned & Poisoned & Poisoned & Poisoned & Poisoned \\
    Patch 6 & Poisoned & Poisoned & Poisoned & Poisoned & Poisoned & Poisoned & Poisoned \\
    Patch 7 & Poisoned & Poisoned & Poisoned & Poisoned & Poisoned & Poisoned & Poisoned \\
    Patch 8 & Poisoned & Poisoned & Poisoned & Poisoned & Poisoned & Poisoned & Poisoned \\
    Patch 9 & Poisoned & Poisoned & Poisoned & Poisoned & Poisoned & Poisoned & Poisoned \\
    Patch 10 & Poisoned & Poisoned & Poisoned & Poisoned & Poisoned & Poisoned & Poisoned \\
    Patch 11 & Poisoned & Poisoned & Poisoned & Poisoned & Poisoned & Poisoned & Poisoned \\
    Patch 12 & Poisoned & Poisoned & Poisoned & Poisoned & Poisoned & Poisoned & Poisoned \\
    Trigger 1 & Poisoned & Poisoned & Poisoned & Poisoned & Poisoned & Poisoned & Poisoned \\
    Trigger 2 & Poisoned & Poisoned & Poisoned & Poisoned & Poisoned & Poisoned & Poisoned \\
    Trigger 3 & Poisoned & Poisoned & Poisoned & Poisoned & Poisoned & Poisoned & Poisoned \\
    Trigger 4 & Poisoned & Poisoned & Poisoned & Poisoned & Poisoned & Poisoned & Poisoned \\
    Trigger 5 & Poisoned & Poisoned & Poisoned & Poisoned & Poisoned & Poisoned & Poisoned \\
    \bottomrule
    \end{tabular}
  }
\end{table}

\begin{table}[H]
  \centering
  \caption{Decisions on the Cifar-10 dataset by each VLM and by majority vote.}
  \label{tab:cifar10_decisions}
  \resizebox{\columnwidth}{!}{%
  \begin{tabular}{@{}l l *{6}{c}@{}}
    \toprule
    \textbf{Image} 
    & \textbf{Type} 
    & \textbf{Grok} 
    & \textbf{Gemini} 
    & \textbf{Claude} 
    & \textbf{ChatGPT} 
    & \textbf{ChatGPT} 
    & \textbf{Majority} \\

    \textbf{Number}
    & \textbf{} 
    & \textbf{(Grok 3)} 
    & \textbf{(2.5 Pro)} 
    & \textbf{(Sonnet~4)} 
    & \textbf{(o4-mini-high)} 
    & \textbf{(GPT\,4.1)} 
    & \textbf{Vote} \\
    \midrule
  
    Original (1) & Clean & Clean & Clean & Clean & Poisoned & Clean & Clean \\
    Original (2) & Clean & Clean & Clean & Clean & Clean & Clean & Clean \\
    Original (3) & Clean & Clean & Clean & Clean & Clean & Clean & Clean \\
    Original (4) & Clean & Poisoned & Poisoned & Clean & Clean & Poisoned & Poisoned \\
    Original (5) & Clean & Clean & Clean & Clean & Clean & Clean & Clean \\
    Original (6) & Clean & Clean & Clean & Clean & Clean & Clean & Clean \\
    Original (7) & Clean & Clean & Clean & Clean & Clean & Clean & Clean \\
    Original (8) & Clean & Clean & Clean & Clean & Clean & Clean & Clean \\
    Original (9) & Clean & Clean & Clean & Clean & Clean & Clean & Clean \\
    Original (10) & Clean & Clean & Clean & Clean & Clean & Clean & Clean \\
    Original (11) & Clean & Clean & Clean & Clean & Clean & Clean & Clean \\
    Original (12) & Clean & Clean & Clean & Clean & Clean & Clean & Clean \\
    Original (13) & Clean & Clean & Clean & Clean & Clean & Clean & Clean \\
    Original (14) & Clean & Clean & Clean & Clean & Clean & Clean & Clean \\
    Original (15) & Clean & Clean & Clean & Clean & Clean & Clean & Clean \\
    Original (16) & Clean & Clean & Clean & Clean & Clean & Clean & Clean \\
    Original (17) & Clean & Clean & Clean & Clean & Clean & Clean & Clean \\
    Original (18) & Clean & Clean & Clean & Clean & Clean & Clean & Clean \\
    Original (19) & Clean & Clean & Clean & Clean & Clean & Clean & Clean \\
    Original (20) & Clean & Clean & Poisoedn & Clean & Clean & Clean & Clean \\
    Original (21) & Clean & Clean & Clean & Clean & Clean & Clean & Clean \\
    Original (22) & Clean & Poisoned & Poisoned & Clean & Clean & Clean & Clean \\
    Original (23) & Clean & Clean & Poisoned & Clean & Clean & Clean & Clean \\
    Original (24) & Clean & Clean & Clean & Clean & Clean & Clean & Clean \\
    Original (25) & Clean & Clean & Clean & Clean & Clean & Clean & Clean \\
    Original (26) & Clean & Clean & Clean & Clean & Clean & Clean & Clean \\
    Original (27) & Clean & Clean & Poisoned & Clean & Clean & Clean & Clean \\
    Original (28) & Clean & Clean & Clean & Clean & Clean & Clean & Clean \\
    Original (29) & Clean & Clean & Clean & Clean & Clean & Clean & Clean \\
    Original (30) & Clean & Clean & Clean & Clean & Clean & Clean & Clean \\
    Original (31) & Clean & Clean & Poisoned & Clean & Clean & Clean & Clean \\
    Original (32) & Clean & Clean & Clean & Clean & Clean & Clean & Clean \\
    Original (33) & Clean & Clean & Clean & Clean & Clean & Clean & Clean \\
    Original (34) & Clean & Clean & Clean & Clean & Clean & Clean & Clean \\
    Original (35) & Clean & Clean & Clean & Clean & Clean & Clean & Clean \\
    Original (36) & Clean & Clean & Clean & Clean & Clean & Clean & Clean \\
    Original (37) & Clean & Clean & Poisoned & Clean & Clean & Clean & Clean \\
    Original (38) & Clean & Clean & Clean & Clean & Clean & Clean & Clean \\
    Original (39) & Clean & Clean & Poisoned & Clean & Clean & Clean & Clean \\
    Original (40) & Clean & Poisoned & Poisoned & Clean & Clean & Clean & Clean \\
    Original (41) & Clean & Clean & Clean & Clean & Clean & Clean & Clean \\
    Original (42) & Clean & Clean & Clean & Clean & Clean & Clean & Clean \\
    Original (43) & Clean & Clean & Clean & Clean & Clean & Clean & Clean \\
    Original (44) & Clean & Clean & Clean & Clean & Clean & Clean & Clean \\
    Original (45) & Clean & Clean & Poisoned & Clean & Clean & Clean & Clean \\
    Original (46) & Clean & Clean & Poisoned & Clean & Clean & Clean & Clean \\
    Original (47) & Clean & Clean & Clean & Clean & Clean & Clean & Clean \\
    Original (48) & Clean & Clean & Clean & Clean & Clean & Clean & Clean \\
    Original (49) & Clean & Clean & Clean & Clean & Clean & Clean & Clean \\
    Original (50) & Clean & Clean & Poisoned & Clean & Clean & Clean & Clean \\
    Original (51) & Clean & Clean & Clean & Clean & Clean & Clean & Clean \\
    Original (52) & Clean & Clean & Poisoned & Clean & Clean & Clean & Clean \\
    Original (53) & Clean & Clean & Poisoned & Clean & Clean & Clean & Clean \\
    Original (54) & Clean & Clean & Clean & Clean & Clean & Clean & Clean \\
    Original (55) & Clean & Clean & Clean & Clean & Clean & Clean & Clean \\
    Original (56) & Clean & Clean & Clean & Clean & Clean & Clean & Clean \\
    Original (57) & Clean & Clean & Clean & Clean & Clean & Clean & Clean \\
    Original (58) & Clean & Clean & Poisoned & Clean & Clean & Clean & Clean \\
    Original (59) & Clean & Clean & Clean & Clean & Clean & Clean & Clean \\
    Original (60) & Clean & Clean & Clean & Clean & Clean & Clean & Clean \\
    Original (61) & Clean & Clean & Clean & Clean & Clean & Clean & Clean \\
    Original (62) & Clean & Clean & Clean & Clean & Clean & Clean & Clean \\
    Original (63) & Clean & Clean & Poisoned & Clean & Clean & Clean & Clean \\
    Original (64) & Clean & Clean & Poisoned & Clean & Clean & Clean & Clean \\
    Original (65) & Clean & Clean & Poisoned & Clean & Clean & Clean & Clean \\
    Original (66) & Clean & Clean & Poisoned & Clean & Clean & Clean & Clean \\
    Original (67) & Clean & Clean & Clean & Clean & Clean & Clean & Clean \\
    Original (68) & Clean & Clean & Clean & Clean & Clean & Clean & Clean \\
    Original (69) & Clean & Clean & Poisoned & Clean & Clean & Clean & Clean \\
    Original (70) & Clean & Clean & Clean & Clean & Clean & Clean & Clean \\
    Original (71) & Clean & Clean & Clean & Clean & Clean & Clean & Clean \\
    Original (72) & Clean & Clean & Poisoned & Clean & Clean & Clean & Clean \\
    Original (73) & Clean & Clean & Clean & Clean & Clean & Clean & Clean \\
    Original (74) & Clean & Clean & Clean & Clean & Clean & Clean & Clean \\
    Original (75) & Clean & Clean & Clean & Clean & Clean & Clean & Clean \\
    Original (76) & Clean & Clean & Clean & Clean & Clean & Clean & Clean \\
    Original (77) & Clean & Clean & Clean & Clean & Clean & Clean & Clean \\
    Original (78) & Clean & Clean & Clean & Clean & Clean & Clean & Clean \\
    Original (79) & Clean & Clean & Clean & Clean & Clean & Clean & Clean \\
    Original (80) & Clean & Poisoned & Poisoned & Clean & Clean & Poisoned & Poisoned \\
    Original (81) & Clean & Clean & Clean & Clean & Clean & Clean & Clean \\
    Original (82) & Clean & Clean & Clean & Clean & Clean & Clean & Clean \\
    Original (83) & Clean & Clean & Poisoned & Clean & Clean & Clean & Clean \\
    Original (84) & Clean & Clean & Poisoned & Clean & Clean & Clean & Clean \\
    Original (85) & Clean & Clean & Clean & Clean & Clean & Clean & Clean \\
    Original (86) & Clean & Clean & Poisoned & Clean & Clean & Clean & Clean \\
    Original (87) & Clean & Clean & Poisoned & Clean & Clean & Clean & Clean \\
    Original (88) & Clean & Clean & Clean & Clean & Clean & Clean & Clean \\
    Original (89) & Clean & Clean & Poisoned & Clean & Clean & Clean & Clean \\
    Original (90) & Clean & Clean & Clean & Clean & Clean & Clean & Clean \\
    Poison (1) & Poisoned & Poisoned & Poisoned & Clean & Poisoned & Poisoned & Poisoned \\
    Poison (2) & Poisoned & Poisoned & Poisoned & Clean & Poisoned & Poisoned & Poisoned \\
    Poison (3) & Poisoned & Poisoned & Poisoned & Clean & Poisoned & Poisoned & Poisoned \\
    Poison (4) & Poisoned & Poisoned & Poisoned & Clean & Poisoned & Poisoned & Poisoned \\
    Poison (5) & Poisoned & Poisoned & Poisoned & Clean & Poisoned & Poisoned & Poisoned \\
    Poison (6) & Poisoned & Poisoned & Poisoned & Clean & Poisoned & Poisoned & Poisoned \\
    Poison (7) & Poisoned & Poisoned & Poisoned & Clean & Poisoned & Poisoned & Poisoned \\
    Poison (8) & Poisoned & Poisoned & Poisoned & Clean & Poisoned & Poisoned & Poisoned \\
    Poison (9) & Poisoned & Poisoned & Poisoned & Clean & Poisoned & Poisoned & Poisoned \\
    Poison (10) & Poisoned & Poisoned & Poisoned & Clean & Poisoned & Poisoned & Poisoned \\
    \bottomrule
  \end{tabular}
  }
\end{table}

\begin{table}[H]
  \centering
  \caption{Confusion metrics for each VLM and for the majority‐vote on the PubFig dataset.}
  \label{tab:confusion_metrics_pubfig}
  \resizebox{\columnwidth}{!}{%
  \begin{tabular}{@{} l c c c c c c @{}}
    \toprule
    \textbf{Metric} 
    & \textbf{Grok} 
    & \textbf{Gemini} 
    & \textbf{Claude} 
    & \textbf{ChatGPT} 
    & \textbf{ChatGPT} 
    & \textbf{Majority} \\

  \textbf{} 
    & \textbf{(Grok 3)} 
    & \textbf{(2.5 Pro)} 
    & \textbf{(Sonnet~4)} 
    & \textbf{(o4-mini-high)} 
    & \textbf{(GPT\,4.1)} 
    & \textbf{Vote} \\
    \midrule
    True Positive\\(Clean $\to$ Clean) 
      & 78 & 55 & 47 & 89 & 90 & 85 \\
    True Negative\\(Poison $\to$ Poison) 
      &  8 & 10 &  7 & 10 & 10 & 10 \\
    False Positive\\(Clean $\to$ Poison) 
      & 12 & 35 & 43 &  1 &  0 &  5 \\
    False Negative\\(Poison $\to$ Clean) 
      &  2 &  0 &  3 &  0 &  0 &  0 \\
    \bottomrule
  \end{tabular}
  }
\end{table}

\begin{table}[H]
  \centering
  \caption{Confusion‐metric percentages for each VLM and the majority vote on the PubFig dataset.}
  \label{tab:confusion_percentages_pubfig}
  \resizebox{\columnwidth}{!}{%
  \begin{tabular}{@{} l *{6}{c} @{}}
    \toprule
    & \multicolumn{5}{c}{\textbf{In Percentage}} \\[-0.5ex]
    \midrule
    \textbf{Metric} 
    & \textbf{Grok} 
    & \textbf{Gemini} 
    & \textbf{Claude} 
    & \textbf{ChatGPT} 
    & \textbf{ChatGPT} 
    & \textbf{Majority} \\

  \textbf{} 
    & \textbf{(Grok 3)} 
    & \textbf{(2.5 Pro)} 
    & \textbf{(Sonnet~4)} 
    & \textbf{(o4-mini-high)} 
    & \textbf{(GPT\,4.1)} 
    & \textbf{Vote} \\
    \midrule
    True Positive\\(Clean $\to$ Clean) 
      & 86.67 & 61.11 & 52.22 & 98.89 & 100.00 & 94.44 \\
    True Negative\\(Poison $\to$ Poison) 
      & 80.00 &100.00 & 70.00 &100.00 &100.00 &100.00 \\
    False Positive\\(Clean $\to$ Poison) 
      & 13.33 & 38.89 & 47.78 &  1.11 &   0.00 &  5.56 \\
    False Negative\\(Poison $\to$ Clean) 
      & 20.00 &  0.00 & 30.00 &  0.00 &   0.00 &  0.00 \\
    \bottomrule
  \end{tabular}
  }
\end{table}

\begin{table}[H]
  \centering
  \caption{Confusion metrics for each VLM and for the majority‐vote on the LFW dataset.}
  \label{tab:confusion_metrics_lfw}
  \resizebox{\columnwidth}{!}{%
  \begin{tabular}{@{} l c c c c c c @{}}
    \toprule
    \textbf{Metric} 
    & \textbf{Grok} 
    & \textbf{Gemini} 
    & \textbf{Claude} 
    & \textbf{ChatGPT} 
    & \textbf{ChatGPT} 
    & \textbf{Majority} \\

  \textbf{} 
    & \textbf{(Grok 3)} 
    & \textbf{(2.5 Pro)} 
    & \textbf{(Sonnet~4)} 
    & \textbf{(o4-mini-high)} 
    & \textbf{(GPT\,4.1)} 
    & \textbf{Vote} \\
    \midrule
    True Positive\\(Clean $\to$ Clean) 
      & 66 & 61 & 77 & 83 & 83 & 82 \\
    True Negative\\(Poison $\to$ Poison) 
      &  17 & 17 &  17 & 17 & 17 & 17 \\
    False Positive\\(Clean $\to$ Poison) 
      & 17 & 22 & 6 &  0 &  0 &  1 \\
    False Negative\\(Poison $\to$ Clean) 
      &  0 &  0 &  0 &  0 &  0 &  0 \\
    \bottomrule
  \end{tabular}
  }
\end{table}

\begin{table}[H]
  \centering
  \caption{Confusion‐metric percentages for each VLM and the majority vote on the LFW dataset.}
  \label{tab:confusion_percentages_lfw}
  \resizebox{\columnwidth}{!}{%
  \begin{tabular}{@{} l *{6}{c} @{}}
    \toprule
    & \multicolumn{5}{c}{\textbf{In Percentage}} \\[-0.5ex]
    \midrule
    \textbf{Metric} 
    & \textbf{Grok} 
    & \textbf{Gemini} 
    & \textbf{Claude} 
    & \textbf{ChatGPT} 
    & \textbf{ChatGPT} 
    & \textbf{Majority} \\

  \textbf{} 
    & \textbf{(Grok 3)} 
    & \textbf{(2.5 Pro)} 
    & \textbf{(Sonnet~4)} 
    & \textbf{(o4-mini-high)} 
    & \textbf{(GPT\,4.1)} 
    & \textbf{Vote} \\
    \midrule
    True Positive\\(Clean $\to$ Clean) 
      & 79.52 & 73.49 & 92.77 & 100.00 & 100.00 & 98.8 \\
    True Negative\\(Poison $\to$ Poison) 
      & 100.00 & 100.00 & 100.00 & 100.00 & 100.00 & 100.00 \\
    False Positive\\(Clean $\to$ Poison) 
      & 20.48 & 26.5 & 7.23 &  0.00 &   0.00 &  1.20 \\
    False Negative\\(Poison $\to$ Clean) 
      & 0.00 &  0.00 & 0.00 &  0.00 &   0.00 &  0.00 \\
    \bottomrule
  \end{tabular}
  }
\end{table}

\begin{table}[H]
  \centering
  \caption{Confusion metrics for each VLM and for the majority‐vote on the Cifar-10 dataset.}
  \label{tab:confusion_metrics_cifar10}
  \resizebox{\columnwidth}{!}{%
  \begin{tabular}{@{} l c c c c c c @{}}
    \toprule
    \textbf{Metric} 
    & \textbf{Grok} 
    & \textbf{Gemini} 
    & \textbf{Claude} 
    & \textbf{ChatGPT} 
    & \textbf{ChatGPT} 
    & \textbf{Majority} \\

  \textbf{} 
    & \textbf{(Grok 3)} 
    & \textbf{(2.5 Pro)} 
    & \textbf{(Sonnet~4)} 
    & \textbf{(o4-mini-high)} 
    & \textbf{(GPT\,4.1)} 
    & \textbf{Vote} \\
    \midrule
    True Positive\\(Clean $\to$ Clean) 
      & 86 & 63 & 90 & 89 & 88 & 88 \\
    True Negative\\(Poison $\to$ Poison) 
      &  10 & 10 &  0 & 10 & 10 & 10 \\
    False Positive\\(Clean $\to$ Poison) 
      & 4 & 27 & 0 &  1 &  2 &  2 \\
    False Negative\\(Poison $\to$ Clean) 
      &  0 &  0 &  10 &  0 &  0 &  0 \\
    \bottomrule
  \end{tabular}
  }
\end{table}

\begin{table*}[]
  \centering
  \caption{Per‑dataset identification counts for each VLM and majority vote.}
  \label{tab:ident_counts}
  \resizebox{\textwidth}{!}{%
    \begin{tabular}{@{} l c c c *{5}{cccc} @{}} 
      \toprule
      \textbf{Dataset} 
        & \textbf{Total} & \textbf{Poisoned} & \textbf{Clean}
        & \multicolumn{4}{c}{\textbf{Grok}} 
        & \multicolumn{4}{c}{\textbf{Gemini}} 
        & \multicolumn{4}{c}{\textbf{Claude}} 
        & \multicolumn{4}{c}{\textbf{ChatGPT (o4‑mini‑high)}} 
        & \multicolumn{4}{c}{\textbf{ChatGPT (GPT‑4.1)}} \\
      \cmidrule(lr){5-8}\cmidrule(lr){9-12}\cmidrule(lr){13-16}
      \cmidrule(lr){17-20}\cmidrule(l){21-24}
        &  &  &  
        & \textbf{CP} & \textbf{IP} & \textbf{CC} & \textbf{IC}
        & \textbf{CP} & \textbf{IP} & \textbf{CC} & \textbf{IC}
        & \textbf{CP} & \textbf{IP} & \textbf{CC} & \textbf{IC}
        & \textbf{CP} & \textbf{IP} & \textbf{CC} & \textbf{IC}
        & \textbf{CP} & \textbf{IP} & \textbf{CC} & \textbf{IC} \\
      \midrule
      PubFig   
        & 100 & 10 &  90 
        &   8 &  2 & 78 & 12 
        &  10 &  0 & 55 & 35 
        &   7 &  3 & 47 & 43 
        &  10 &  0 & 89 &  1 
        &  10 &  0 & 90 &  0 \\
      LFW      
        & 100 & 17 &  83 
        &  17 &  0 & 66 & 17 
        &  17 &  0 & 61 & 22 
        &  17 &  0 & 77 &  6 
        &  17 &  0 & 83 &  0 
        &  17 &  0 & 83 &  0 \\
      CIFAR‑10 
        & 100 & 10 &  90 
        &  10 &  0 & 86 &  4 
        &  10 &  0 & 63 & 27 
        &   0 & 10 & 90 &  0 
        &  10 &  0 & 89 &  1 
        &  10 &  0 & 88 &  2 \\
      \bottomrule
    \end{tabular}%
  }
  
  \begin{minipage}{.9\textwidth}
    \footnotesize
    \textbf{CP}: Correctly Identified Poisoned\quad
    \textbf{IP}: Incorrectly Identified Poisoned\quad
    \textbf{CC}: Correctly Identified Clean\quad
    \textbf{IC}: Incorrectly Identified Clean
  \end{minipage}
\end{table*}

\begin{table*}[]
\centering
\caption{Summary of image regeneration feasibility experiments across datasets and VLMs.}
\label{tab:llm_regeneration_summary}
 \resizebox{\textwidth}{!}{
\begin{tabular}{lccc}
\toprule
\textbf{Model} & \textbf{PubFig Dataset (Face)} & \textbf{LFW Dataset (Face)} & \textbf{CIFAR-10 Dataset(Object)} \\
\midrule
Grok & Regenerated incorrect images & Regenerated incorrect images & Regenerated incorrect images \\
Claude & Denied regeneration (privacy reason) & Denied regeneration (privacy reason) & Denied regeneration (privacy reason) \\ 
Gemini & Denied regeneration (capability limitation) & Denied regeneration (capability limitation) & Denied regeneration (capability limitation) \\
ChatGPT (o4-mini-high) & Regenerated incorrect or null images & Regenerated incorrect or null images & Regenerated incorrect or null images \\
ChatGPT (GPT-4.1) & Denied regeneration (privacy/capability limitation) & Denied regeneration (privacy/capability limitation) & Partial regeneration with patch removal or \\
 & & & null images or denial\\
\bottomrule
\end{tabular}
}
\end{table*}

\begin{table}[H]
  \centering
  \caption{Confusion‐metric percentages for each VLM and the majority vote on the Cifar-10 dataset.}
  \label{tab:confusion_percentages_cifar10}
  \resizebox{\columnwidth}{!}{
  \begin{tabular}{@{} l *{6}{c} @{}}
    \toprule
    & \multicolumn{5}{c}{\textbf{In Percentage}} \\[-0.5ex]
    \midrule
    \textbf{Metric} 
    & \textbf{Grok} 
    & \textbf{Gemini} 
    & \textbf{Claude} 
    & \textbf{ChatGPT} 
    & \textbf{ChatGPT} 
    & \textbf{Majority} \\

  \textbf{} 
    & \textbf{(Grok 3)} 
    & \textbf{(2.5 Pro)} 
    & \textbf{(Sonnet~4)} 
    & \textbf{(o4-mini-high)} 
    & \textbf{(GPT\,4.1)} 
    & \textbf{Vote} \\
    \midrule
    True Positive\\(Clean $\to$ Clean) 
      & 95.56 & 70.00 & 100.00 & 98.89 & 97.78 & 97.78 \\
    True Negative\\(Poison $\to$ Poison) 
      & 100.00 & 100.00 & 0.00 & 100.00 & 100.00 & 100.00 \\
    False Positive\\(Clean $\to$ Poison) 
      & 4.44 & 30.00 & 0.00 &  1.11 &   2.22 &  2.22 \\
    False Negative\\(Poison $\to$ Clean) 
      & 0.00 &  0.00 & 100.00 &  0.00 &   0.00 &  0.00 \\
    \bottomrule
  \end{tabular}
  }
\end{table}

\begin{figure}[H]
    \centering
    \includegraphics[width=0.45\textwidth]{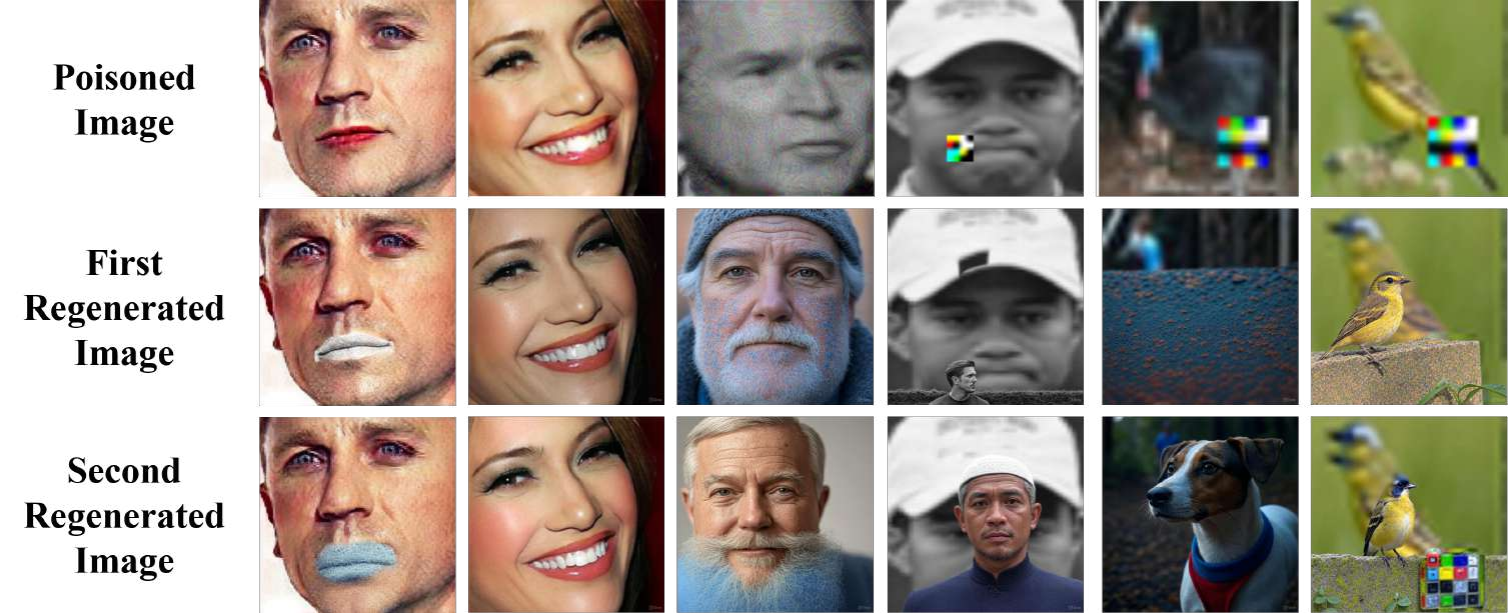}
    \caption{Clean images regenerated by Grok. Grok regenerated two clean images for each poisoned image.}
    \label{fig24}
\end{figure}

\begin{figure}[H]
    \centering
    \includegraphics[width=0.25\textwidth]{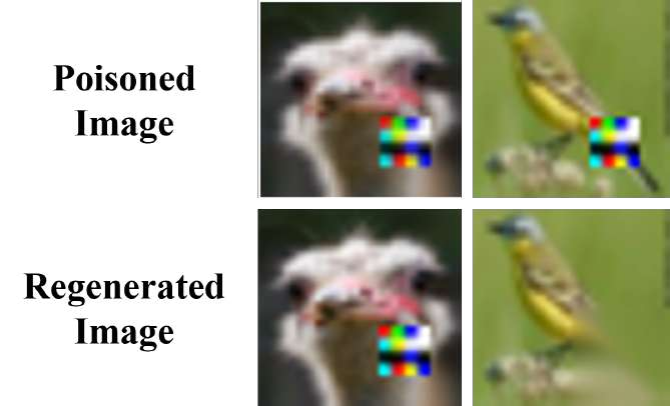}
    \caption{Clean images regenerated by ChatGPT GPT 4.1.}
    \label{fig26}
\end{figure}

\begin{figure}[H]
    \centering
    \includegraphics[width=0.45\textwidth]{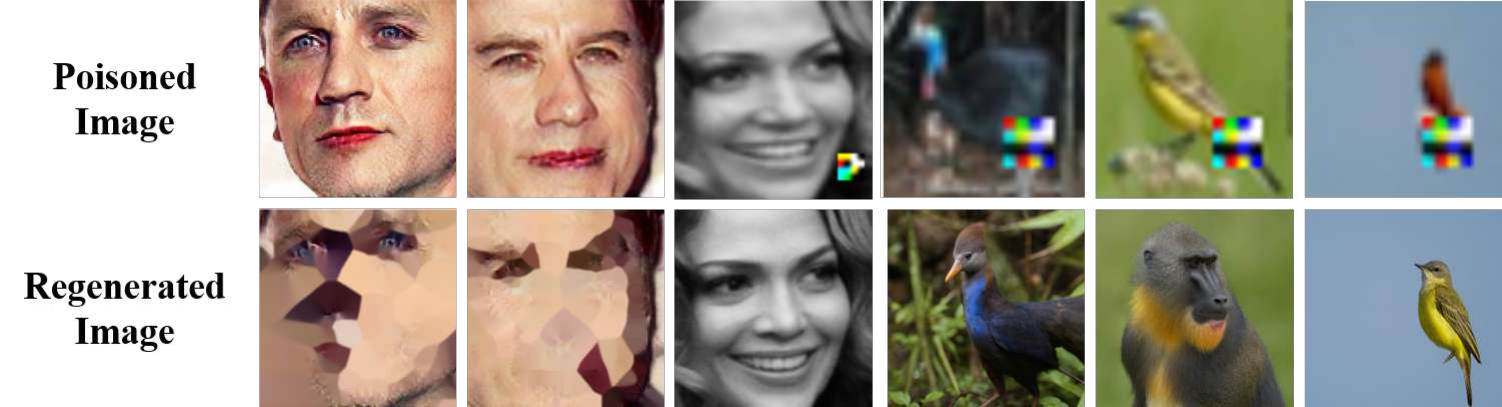}
    \caption{Clean images regenerated by ChatGPT o4-mini-high.}
    \label{fig25}
\end{figure}

\begin{figure}[H]
    \centering
    \includegraphics[width=0.4\textwidth]{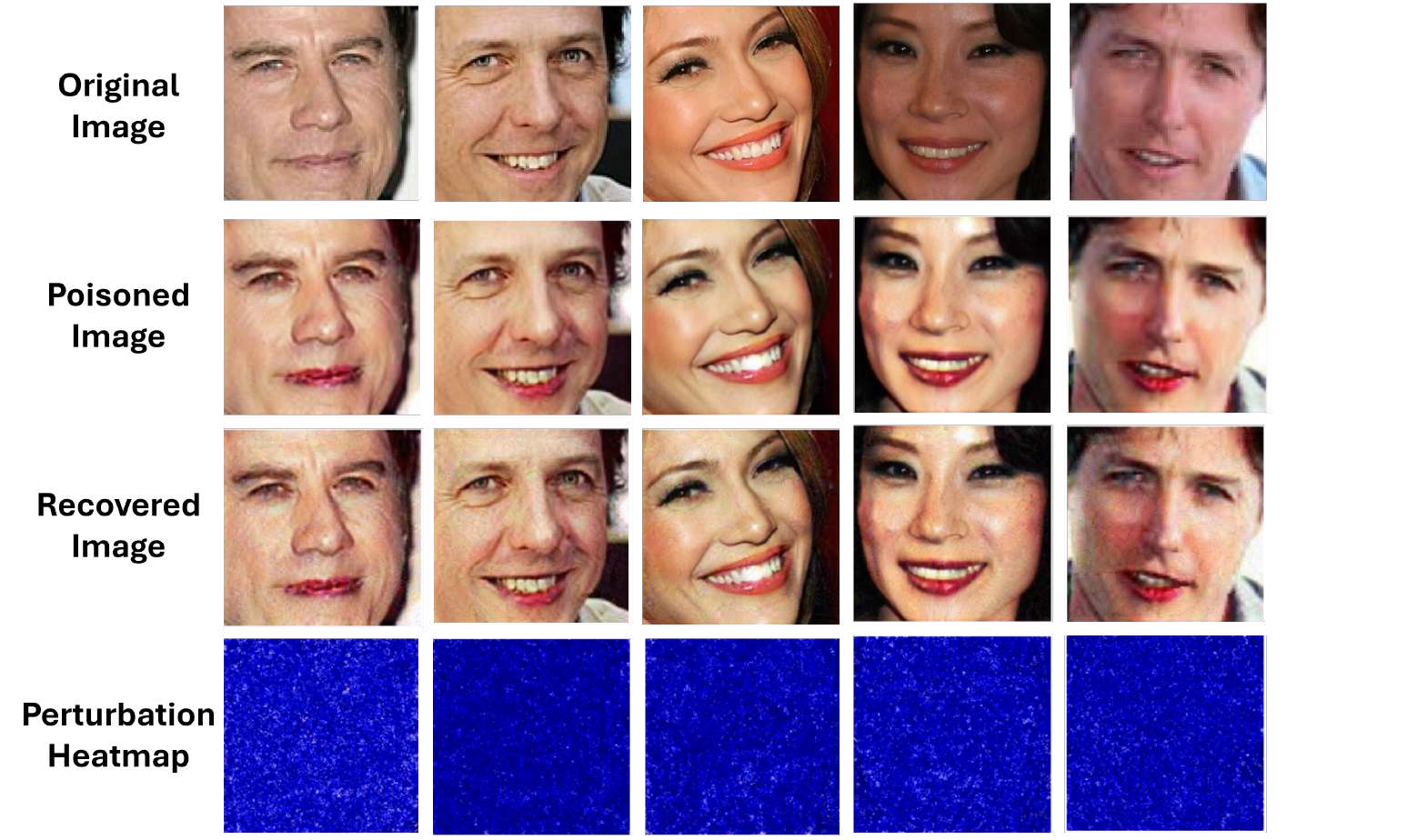}
    \caption{Additional examples of recovered images and perturbation heatmaps from the PubFig dataset.}
    \label{fig19}
\end{figure}

\begin{figure}[H]
    \centering
    \includegraphics[width=0.5\textwidth]{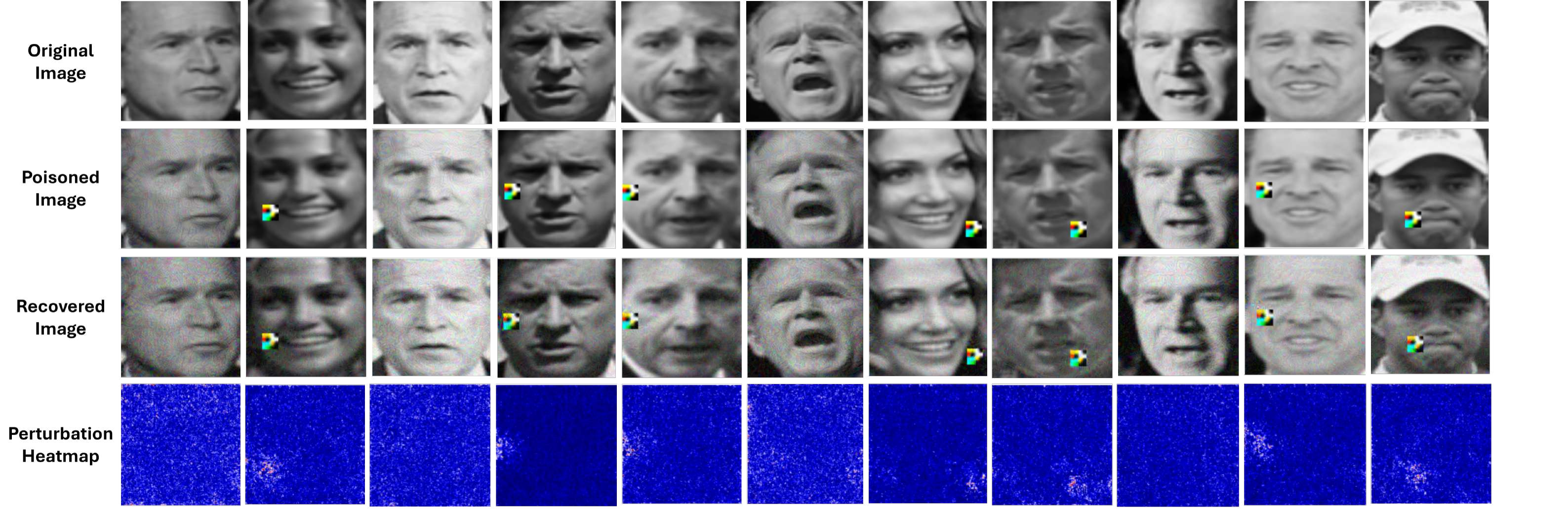}
    \caption{Additional examples of recovered images and perturbation heatmaps from the LFW dataset.}
    \label{fig20}
\end{figure}

\begin{figure}[H]
    \centering
    \includegraphics[width=0.4\textwidth]{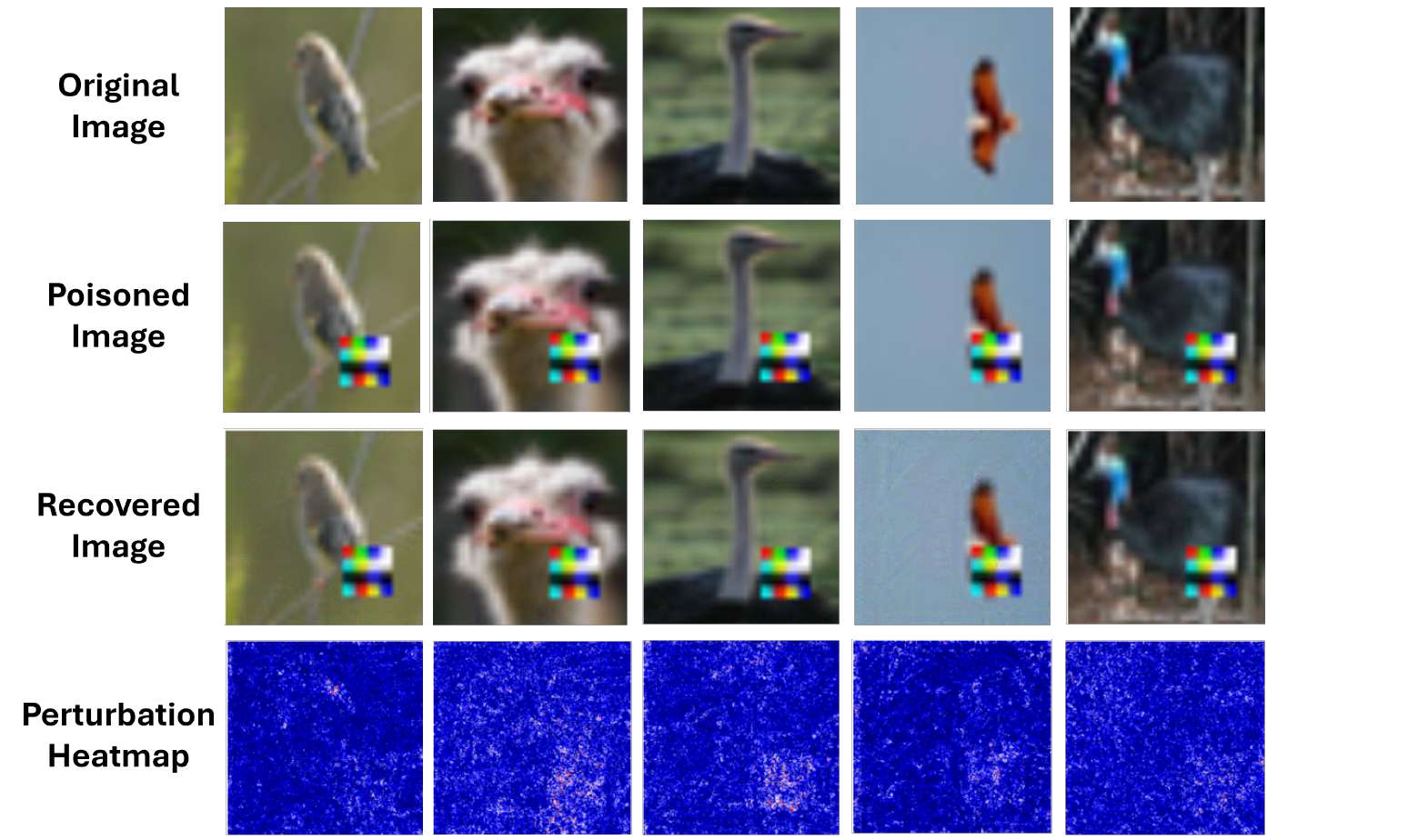}
    \caption{Additional examples of recovered images and perturbation heatmaps from the Cifar-10 dataset.}
    \label{fig21}
\end{figure}

\begin{table}[H]
\centering
\caption{Mathematical notation used throughout the paper}
\label{tab:notation}
 \resizebox{\columnwidth}{!}{
\begin{tabular}{cl}
\hline
    \textbf{Notation} & \textbf{Description} \\
    \hline

    \multicolumn{2}{c}{\textbf{Data Poisoning}}\\
    \hline
    $s$ & Source image from the attacker class \\
    $t$ & Target image from the victim class \\
    $p$ & Visual trigger or makeup pattern \\
    $\tilde{s}$ & Patched version of the source image \\
    $z$ & Poisoned image optimized in feature space \\
    $f(\cdot)$ & Feature extractor of the deep model \\
    $x^s$ & Source image in MakeupAttack method \\
    $x^p$ & Poisoned image \\
    $m$ & Binary mask for applying makeup \\
    $\odot$ & Element-wise multiplication \\
    \hline

    \multicolumn{2}{c}{\textbf{Multimodal Voting Based Mechanism}}\\
    \hline
    $\mathcal{D}_{\mathrm{untrusted}}$ & Training dataset containing poisoned samples \\
    $x_i$ & An image in the mixed dataset \\
    $\mathcal{V}$ & Set of Vision-Language Models \\
    $V_j$ & The $j$-th Vision-Language Model \\
    $v_j(x)$ & Decision of $V_j$ for image $x$ \\
    $v$ & Number of models that predicted the image as poisoned \\
    $\hat{y}(x)$ & Final prediction for image $x$ \\
    $\mathbf{1}\{\cdot\}$ & Indicator function returning 1 if condition is true \\
    \hline

    \multicolumn{2}{c}{\textbf{Corrective PGD--Based Recovery}}\\
    \hline
    $y^p$ & Ground-truth label of the poisoned image \\
    $f_\theta$ & Trained classifier with parameters $\theta$ \\
    $\tilde{x}^p$ & Corrected (recovered) version of the poisoned image \\
    $\delta_{\text{trig}}$ & Additive backdoor trigger \\
    $\Delta$ & Upper bound on the trigger strength \\
    $\epsilon$ & Final corrective noise budget \\
    $\alpha$ & PGD step size \\
    $T$ & Number of PGD iterations \\
    $x^{(t)}$ & PGD intermediate image at iteration $t$ \\
    $Proj_{x^p, \epsilon}[\cdot]$ & Projection operator onto $\ell_\infty$-ball around $x^p$ \\
    $\mathcal{L}(\cdot)$ & Loss function used for PGD updates \\
    $\rho_\infty^p$ & $\ell_\infty$ noise magnitude between $x^p$ and $\tilde{x}^p$ \\
    $\rho_2^p$ & $\ell_2$ noise magnitude between $x^p$ and $\tilde{x}^p$ \\
    $\delta$ & Safety margin scalar for scaling $\Delta_{\max}$ \\
    $\Delta_{\max}$ & Maximum observed trigger magnitude over poisoned set \\
    $x_{\text{adv}}$ & PGD output with full-range budget $\epsilon = 1.0$ \\
    $\mathcal{B}_\infty(x, \epsilon)$ & $\ell_\infty$-ball of radius $\epsilon$ centered at $x$ \\
    $\delta^*$ & Corrective perturbation vector satisfying robustness conditions \\
    $g$ & Gradient of the loss with respect to input\\
    $clip(\cdot)$ & Function that bounds values within a given range\\
    $\mathcal{D}_{\mathrm{clean}}$ & Final cleaned dataset after applying detection and recovery \\
    \hline
\end{tabular}
}
\end{table}

\end{document}